\documentclass{article}

\usepackage{natbib}
\setcitestyle{numbers, compress}

\usepackage{amsmath}
\usepackage{amssymb}
\usepackage{mathtools}
\usepackage{amsthm}
\usepackage{bbding}
\theoremstyle{plain}
\newtheorem{theorem}{Theorem}[section]
\newtheorem{proposition}[theorem]{Proposition}
\newtheorem{lemma}[theorem]{Lemma}

\theoremstyle{definition}
\newtheorem{definition}[theorem]{Definition}
\newtheorem{assumption}[theorem]{Assumption}
\newtheorem{remark}[theorem]{Remark}
\newtheorem*{main generalization result}{Main Generalization Result}
\allowdisplaybreaks[4]

\usepackage[preprint]{neurips_2022}

\usepackage[utf8]{inputenc} 
\usepackage[T1]{fontenc}    
\usepackage{hyperref}       
\usepackage{url}            
\usepackage{booktabs}       
\usepackage{amsfonts}       
\usepackage{nicefrac}       
\usepackage{microtype}      
\usepackage{xcolor}         

\hypersetup{colorlinks=true,linkcolor=blue,linktocpage=false,citebordercolor=blue,citecolor=blue,anchorcolor=blue}

\usepackage{subfigure}
\usepackage{etoc}

\title{Generalization Error Bounds for Deep Neural Networks Trained by SGD}

\author{%
  Mingze Wang \\
  School of Mathematical Sciences \\
  Peking University\\
  Beijing, 100081, P.R. China \\
  \texttt{mingzewang@stu.pku.edu.cn} \\
  \And
  Chao Ma\\
  Department of Mathematics \\
  Stanford University\\
  Stanford, CA 94305 \\
  \texttt{chaoma@stanford.edu} \\
}

\begin{document}

\maketitle

\begin{abstract}
Generalization error bounds for deep neural networks trained by stochastic gradient descent (SGD) are derived by combining a dynamical control of an appropriate parameter norm and the Rademacher complexity estimate based on parameter norms. The bounds explicitly depend on the loss along the training trajectory, and work for a wide range of network architectures including multilayer perceptron (MLP) and convolutional neural networks (CNN). Compared with other algorithm-depending generalization estimates such as uniform stability-based bounds, our bounds do not require $L$-smoothness of the nonconvex loss function, and apply directly to SGD instead of Stochastic Langevin gradient descent (SGLD). Numerical results show that our bounds are non-vacuous and robust with the change of optimizer and network hyperparameters. 
\end{abstract}

\section{Introduction}
Deep neural networks (DNN) trained by optimization algorithms based on Stochastic Gradient Descent (SGD) have achieved great performance in various fields such as computer vision, natural language processing, and speech recognition~\citep{goodfellow2016deep}. Yet, theoretical understanding for the surprising generalization capability of DNNs still has a long way to go to explain the success under practical settings \citep{zhangrethinking2017}. Along this direction, the main hurdles are the over-parameterization and the strong algorithm dependency. On one hand, over-parameterized neural networks have super rich hypothesis spaces that can perfectly interpolation all training data, which hinder the application of traditional complexity theories such as the VC-dimension. On the other hand, the hypothesis explored by these networks depend sensitively and dynamically on the optimization algorithm and its hyperparameters, which makes it hard to isolate the models from algorithms in a generalization theory. Therefore, it is crucial to study how optimization algorithms narrow down the hypothesis space and benefit the generalization performance. 

Some techniques are developed to address one or two of the obstacles mentioned above, and provide non-vacuous bounds of generalization errors for deep neural networks. Such works include parameter norm-based \citep{golowich2018size, ma2018priori, weinan2021barron, bartlett2017spectrally} and uniform stability-based estimates \citep{rogers1978finite, hardt2016train, bousquet2020sharper, hoffer2017train}. However, the works either fall short to consider the algorithm dependency (like norm-based bounds), or need to impose strong conditions on the algorithm and loss function (like the isotropic noise and $L$-smooth loss function for stability-based bounds). In this paper, instead, we derive a class of generalization error bounds that are algorithm dependent, and work in much more realistic settings. Technically, we combine an analysis of SGD trajectory with the norm-based generalization error estimates, and the only assumption for our analysis is the boundedness of the network function. Our analysis can be applied to a wide range of network structures, such as fully-connected neural networks (FNNs) and convolutional neural networks (CNNs), and a wide range of optimization algorithms, including GD, SGD and SGLD. 

\subsection{Main Results}
An illustrative description of our main generalization bound is as follows:
\begin{main generalization result}[\bfseries Informal]\ \\
Consider deep FNNs or CNNs trained by algorithms such as GD and SGD with quadratic loss. Let $n$ be the number of training data and $L$ be the depth of networks. Assume the output of the neural network model is uniformly bounded. Then, with high probability we have:
\[
{\rm generalization\ error}\lesssim \mathcal{O}\Big(\frac{{\rm cumulative\ loss}^{L/2}}{\sqrt{n}}\Big).
\]
\end{main generalization result}

{The cumulative loss is a functional of the training loss trajectory. It is smaller when the loss decreases faster during training. This term is slightly different for different networks and algorithms.}

\textbf{Range of applicability.}\\
On the algorithm side, our bounds applies to the practically used version of SGD, going beyond the SGLD studied in most uniform stability-based works. {Compared with the anisotropic noise of SGD, SGLD takes an isotropic noise and has different behaviors.} Our analysis also works on full-batch GD. On the model side, our bounds hold for any layer-wise neural network models with homogeneous activation functions, such as FNNs, CNNs, and RNNs with ReLU or Leaky ReLU activation functions. On these models, our estimates are independent of the neural network's width, and all terms in the bounds are easy to calculate along the training process. More importantly, our bounds do no suffer from the curse of dimensionality. Lastly, while most previous works studying the algorithm-dependent generalization performance of neural networks are built on the seemingly reasonable $L$-smoothness assumption of the loss function (e.g. the uniform stability bounds), our analysis does not rely on this assumption and only need the network function to be bounded. Therefore, our bounds are not impaired by recent works questioning the $L$-smoothness of the loss function~\citep{cohen2021gradient}.

\textbf{Comparison with uniform stability results.}\\
Uniform stability is a representative technique to derive algorithm-dependent generalization bounds based on the algorithm's stability with respect to perturbations on training data~\citep{hoffer2017train}. One problem of the application of uniform stability is its dependency on the $L$-\textit{smoothness} assumption of the loss function, which is imposed in all such works treating nonconvex loss functions and SGD. Recently it is shown that gradient descent on DNNs cannot be analyzed using (even local) $L$-smoothness at any reasonable step size because the sharpness hovers just above $2/\eta_t$  \citep{cohen2021gradient,wu2018sgd}.

On the other side, in the non-$L$-smooth scenario, the uniform stability theory can be applied to analyze Stochastic Langevin Gradient Descent (SGLD) method \citep{bassily2020stability,raginsky2017non,welling2011bayesian,zhang2017hitting} rather than GD or SGD \citep{mou2018generalization}. Though, SGLD is just an approximation to SGD in theory, and it is unclear that whether the isotropic Gaussian noise of SGLD is an appropriate substitution to the anisotropic SGD noise in practice~\citep{zhu2018anisotropic}.

Compared with uniform stability-based results, our bounds works in a much more realistic setting---SGD/GD optimizing a nonconvex loss function without $L$-smoothness condition. A simple comparison on different cases is shown in Table~\ref{sample-table}. We conduct further detailed comparison and discussion in Section~\ref{section:versus}.

\begin{table}[ht]
\caption{Comparison with uniform stability results for deep CNNs or FNNs.}
\label{sample-table}
\vskip 0.1in
\begin{center}
\begin{small}
\begin{tabular}{c|c|c}
      \hline\hline 
      & Bounds & Assumptions \\  \hline
      Uniform Stability for SGLD &  \Checkmark 
	& Bounded and Lipschitz 
	\\\hline
	  Uniform Stability for GD/SGD
	  & \XSolidBrush &\XSolidBrush \\  \hline
	\textbf{Our Results} for (batch) \textbf{GD/SGD}
	 &  \Checkmark & \textbf{Only} Bounded
	\\
	\hline \hline			
\end{tabular}
\end{small}
\end{center}
\vskip -0.1in
\end{table}

\section{\bfseries Related Work}\label{section related work}
Classical statistical learning theories such as Vapnik-Chervonenkis (VC) dimension~\citep{vapnik1994measuring} fails to give effective generalization bounds for DNNs in the over-parameterized scenario~\citep{neyshabur2017exploring}. Researchers have proposed other theories to explain the generalization performance of deep neural networks. We list some such theories below.

\textbf{Norm-based complexity measure.}\ \ 
Norm-based generalization bounds are a class of representative results that use parameter norms to control the Rademacher complexity of the hypothesis space. These bounds do not explicitly depend on the number of parameters. Various parameter norms have been proposed and used, such as path norm \citep{ma2018priori, ma2019priori, li2020complexity, weinan2021barron}, $l^{p,q}$ norm \citep{golowich2018size}, spectral norm \citep{bartlett2017spectrally} and Fisher-Rao norm \citep{liang2019fisher, tu2019understanding}. 

\textbf{Uniform stability theory.}\ \ Uniform stability approach is also extensively used to derive generalization bounds
\citep{rogers1978finite,bousquet2020sharper}. These bounds, depending on the optimization trajectory, often take the ``Train faster, Generalize better'' form~\citep{hardt2016train}.
However, the applications of uniform stability theory often relies on smoothness assumptions of the loss function, while recent work argued that gradient descent on DNNs cannot be analyzed using (even local) $L$-smoothness at reasonable step size~\citep{cohen2021gradient}. 
While attempts are made to bypass the $L$-smoothness condition, so far the analysis can only be conduced on SGLD~\citep{raginsky2017non, welling2011bayesian, zhang2017hitting} rather than real GD or SGD \citep{mou2018generalization}.

\textbf{Other generalization theories.}\ \ Another notable line of works on generalization bounds employs the information theory \citep{shwartz2017opening}. In~\citep{kraskov2004estimating}, mutual information (MI) is used to measure the information transmission and information loss of deep learning models and algorithms. In~\citep{xu2017information, haghifam2020sharpened, bu2020tightening}, MI is used to derive algorithm-dependent generalization bounds. Chaining and conditional MI are also explored to derive more accurate bounds~\citep{asadi2018chaining, steinke2020reasoning}. 
Besides, other techniques and approaches used to bound generalization error include model compression~\citep{arora2018stronger}, margin theory~\citep{li2018tighter}, path length estimate~\citep{liu2022connecting} and linear stability of optimization algorithms~\citep{ma2021sobolev}.

Lastly, we particularly mention the work \citep{liu2022connecting}. This work is related to ours since we both consider the connection between optimization and generalization by path estimate. In~\citep{liu2022connecting}, the authors derive generalization bounds for the Gradient Flow (GF) on linear and nearly linear models. By comparison, we analyze GF, GD and SGD for non-linear deep neural networks (MLPs and CNNs).



\section{Preliminaries}\label{section preliminary}
\subsection{Notations}
We use capital letters to represent vectors or matrices and lowercase letters to represent scalars, e.g. $\mathbf{x}=(x_1,\cdots,x_d)^\top\in\mathbb{R}^d$ and $\mathbf{A}=(A_{ij})_{m_1\times m_2}\in\mathbb{R}^{m_1\times m_2}$. We use $\left<\cdot,\cdot\right>$ to denote the standard Euclidean inner product between
two vectors. $\left\|\cdot\right\|_2$, $\left\|\cdot\right\|_F$, and $\left\|\cdot\right\|_{p,q}$ are $l_2$ norm, Frobenius norm, and $(p,q)$ norm of matrices, respectively, where $\left\|\mathbf{A}\right\|_{p,q}=\max_{\mathbf{x}\ne \mathbf{0}}(\left\|\mathbf{Ax}\right\|_q/\left\|\mathbf{x}\right\|_p)$, and $\left\|\cdot\right\|_p$ is $p$ norm of vectors. We use $\lesssim$ to hide  absolute constants. Let ${\rm vec}(\mathbf{A})$ be the vectorization of a matrix $\mathbf{A}$ in column-first order.
Let $[n]=\{1,\cdots,n\}$. Denote by $\mathcal{N}(\mathbf{0},\mathbf{\Sigma})$ the high dimensional Gaussian distribution with mean $\mathbf{0}$ and covariance $\mathbf{\Sigma}$. 


\subsection{Problem Setup}
In this paper, we consider supervised learning problems. Let $\mu$ be a data distribution.
We are given $n$ training data $\{(\mathbf{x}_i,y_i)\}_{i=1}^n$ drawn i.i.d. from $\mu$. Without loss of generality, we assume $\left\|{\rm vec}(\mathbf{x})\right\|_2\leq1$ and $|y|\leq C_y\leq1$ for $(\mathbf{x},y)$.

In supervised learning, the population risk with quadratic loss and the corresponding empirical risk can be written as
{
\begin{align}
    &\mathcal{L}_{\mu}({\Theta})=\mathbb{E}_{(\mathbf{x},y)\sim\mu}[\ell(\mathbf{x},y;\Theta)],
    \label{population loss}
    \\
    &\mathcal{L}_n({\Theta})=\frac{1}{n}\sum\limits_{i=1}^n\ell(\mathbf{x}_i,y_i;\Theta)\label{empirical loss},
\end{align}}
where $\ell(\mathbf{x},y;\Theta) =\frac{1}{2}(f(\mathbf{x};\Theta)-y)^2$, $f(\mathbf{x},\Theta)$ is the model and $\Theta$ represents all parameters of the model.
The generalization error is defined as:
{
\begin{equation}\label{generalization error}
    \mathcal{E}_{\rm gen}(\Theta)=\mathcal{L}_{\mu}(\Theta)-\mathcal{L}_n(\Theta).
    \end{equation}
}

The learning problem is solved by minimizing the empirical risk~\eqref{empirical loss} using some optimization algorithms such as Gradient Descent Algorithm (GD) and Stochastic Gradient Descent Algorithm (SGD) starting from random initialization.

\subsection{Models}
We consider a general class of deep neural networks as our prediction model $f(\mathbf{x},\Theta)$. The class of models contains widely used deep FNNs and CNNs.
In our models, we consider a normalization factor $1/m^p$ $(p\geq0)$ at the output layer, which allows our following theoretical analysis to cover not only the regular case $(p=0)$, but also the NTK case $(p=1/2)$ \citep{jacot2018neural} and the mean-field case $(p=1)$ \citep{mei2018mean, ma2018priori}.


\noindent\textbf{Deep CNN or FNN.}\ 
We define a neural network with $L_C$ convolutional layers followed by $L_F$ fully-connected layers as follows:
\begin{equation}\label{deep CNN}
    \begin{aligned}
    f(\mathbf{x};\Theta)&=\frac{1}{m^p}\sum_{k=1}^{m}a_k{z}_k^{(L)},\\
	\mathbf{z}^{(l)}&=\sigma(\mathbf{A}^{(l)^{\top}}\mathbf{z}^{(l-1)}),\ l\in[L]-[L_C],\\
	\mathbf{z}^{(l)}&={\rm pool}(\mathbf{y}^{(l)}),\ l\in[L_C], \\
	\mathbf{y}^{(l)}& =\sigma(\mathbf{w}^{(l)^\top}\circledast\mathbf{z}^{(l-1)}), \ l\in[L_C],\\
    \mathbf{z}^{(0)}&=\mathbf{x},
    \end{aligned}
\end{equation}

where $\sigma(z)$ is the ReLU function $\max\{z,0\}$, $\circledast$ is the convolutional operation, pool$(\cdot)$ is the average/max pooling operation, $\mathbf{x}$ is the input, and $m_l$ is the dimension of ${\rm vec}(\mathbf{z}^{(l)})$. Considering the output layer, the depth of such network is $L_C+L_F+1$. When $L_C=0$, this is a fully-connected network. Let $L=L_C+L_F$. For output layer $l=L+1$, let $\Theta^{(L+1)}:=(a_1,\cdots,a_m)^\top\in\mathbb{R}^{m}$. For fully-connected layer $l\in[L]-[L_C]$, $\mathbf{A}^{(l)}\in\mathbb{R}^{m_l \times m_{l-1}}$ and we let $\Theta^{(l)}:={\rm vec}(\mathbf{A}^{(l)})$. For convolution layer $l\in[L_C]$, we consider the structure \texttt{Conv $\to$ ReLU $\to$ Pooling}, and $\Theta^{(l)}:={\rm vec}(\mathbf{w}^{(l)})\in\mathbb{R}^{s_l}$. Then $\Theta=(\Theta^{(1)^\top},\cdots,\Theta^{(L+1)^\top})^\top\in\mathbb{R}^{m+\sum_{l\in[L_C]}s_l+\sum_{l\in[L]-[L_C]}m_{l-1}m_l}$ represents all parameters. We denote the dimension of $\Theta^{(l)}$ as $q(l)$.

\subsection{Optimization Algorithms}
\label{subsection: algorithms}

\textbf{Random initialization.} We use the Gaussian random initialization for each layer:
{
\begin{equation}\label{random initialization}
 \Theta^{(l)}(0)\sim\mathcal{N}(\mathbf{0},\frac{\kappa^2}{q(l)}\mathbf{I}_{q(l)}),\ \forall l\in[L+1],
\end{equation}}
where $\kappa^2=\mathcal{O}(1)$ ($\kappa\neq0$) controls the scale of initialization, and $q(l)$ is the number of parameters in layer $l$. Similar initializations are standard practices in applications~\citep{glorot2010understanding, he2015delving}.

\textbf{Update rules.}
We mainly consider the mini-batch SGD
{
\begin{equation}
{\rm\textbf{SGD}:}\ 
\Theta{(t+1)}=\Theta{(t)}-\frac{\eta_t}{B}\sum\limits_{i=1}^B\nabla \ell(\mathbf{x}_{\gamma_i^t},y_{\gamma_i^t};\Theta{(t)}),
\label{disc SGD}
\end{equation}
}
where $\gamma^t=(\gamma_1^t,\cdots,\gamma_B^t)$ is a $B$-dimensional random variable uniformly distributed on the $B$-tuples
in $[n]$ and independent with $\Theta(t)$.
Our theory also applies to full batch GD, and even the continuous gradient flow (GF) which is the limit of GD as the step size tends to $0$:
\begin{align}
&{\rm\textbf{GF}:}
\ \ \ \ \frac{\mathrm{d}\Theta{(t)}}{\mathrm{d}t}=-\nabla \mathcal{L}(\Theta{(t)}), \label{conti GD}\\
&{\rm\textbf{GD}:}
\ \ \ \ 
\Theta{(t+1)}=\Theta{(t)}-\eta_t\nabla \mathcal{L}(\Theta{(t)}), \label{disc GD}
\end{align}

\section{Generalization Bounds}\label{section generalization bound}
During the analysis, we make the following boundedness assumption. 
Notably, except the boundedness assumption, we \textit{do not need} any other assumption such as Lipschitz continuity and $L$-smoothness of the loss function. 
\begin{assumption}\label{assumption bounded}
We use $\mathcal{F}$ to denote the hypothesis
space, i.e. the set of all output functions from the neural network trained by GD or SGD. 
We assume that there exists $C_f>0$ s.t. $\sup\limits_{f\in\mathcal{F}}\left|f\right|\leq C_f$.
\end{assumption}
The boundedness assumption is necessary for controlling Rademacher complexity and hence is widely used in previous studies on generalization performance \citep{mohri2018foundations, mou2018generalization}, even together with Gaussian initialization \citep{arora2019fine}. When we train neural networks, there are some regularization tricks like scaling outputs into some interval, which ensures the boundedness of hypothesis space. Hence, making theoretical analysis under the boundedness assumption makes sense.

Now, we can state our generalization bounds for deep FNNs or CNNs trained by the algorithms listed in Section \ref{subsection: algorithms}.

\begin{theorem}[\bfseries GF]\label{continuous GD thm}\ Let $\Theta{(t)}$ be trained by GF \eqref{conti GD} with random initialization \eqref{random initialization}. We define the continuous cumulative loss at $T$ as:
\[
{\rm\sf{CL}}(T)=\int_{0}^T2\sqrt{2\mathcal{L}_n(\Theta{(t)})}\Big(C_y-\sqrt{2\mathcal{L}_n(\Theta{(t)})}\Big)\mathrm{d}t.
\]
Then with probability at least $1-\delta$, we have:
\[
     \mathcal{E}_{\rm gen}(\Theta(T))
     \lesssim \frac{C_{L,d}}{m^p\sqrt{n}}\Big(\mathcal{O}(\kappa^2)+{\rm\sf{CL}}(T)\Big)^{\frac{L+1}{2}}+
    \sqrt{\frac{\log(1/\delta)}{n}},
\]
where $C_{L,d}$ is a constant defined in Lemma \ref{Estimation of Rademacher Complexity}.
\end{theorem}

\begin{theorem}[\bfseries GD]\label{thm disc GD 2NN}\
Let $\Theta{(t)}$ be trained by GD \eqref{disc GD} with random initialization \eqref{random initialization}. Let the learning rate $\eta_t=\eta/\lceil\frac{t+1}{T_0}\rceil^{\alpha}$ be satisfied of 
$\alpha\in(\frac{L+1}{L+2},1]$ and 
$\eta=\mathcal{O}\Big(\frac{m^p L^{\frac{L-1}{2}}\sqrt{\epsilon}}{(C_f+C_y)T_0^{\frac{1+\epsilon}{2}}\kappa^{L-1}},
    \frac{\kappa^2(1-\epsilon)}{C_y^2 T_0}
    \Big)$,
where $\epsilon\in(0,1)$. Define the discrete cumulative loss at $T$ as:
\begin{equation}\label{equ: disc CT}
{\rm\sf{CL}}(T) =\sum_{t=0}^{T-1}2\eta_t\sqrt{2\mathcal{L}_n(\Theta{(t)})}\Big(C_y-\sqrt{2\mathcal{L}_n(\Theta{(t)})}\Big).\end{equation}
Then with probability at least $1-\delta$, we have:
\[
\mathcal{E}_{\rm gen}(\Theta(T))
\lesssim
\frac{C_{L,d}}{m^p\sqrt{n}} \Big(
\mathcal{O}(\kappa^2)+{\rm\sf{CL}}(T)\Big)^\frac{L+1}{2}+\sqrt{\frac{\log(1/\delta)}{n}},
\]
where $C_{L,d}$ is defined in Lemma $\ref{Estimation of Rademacher Complexity}$.
\end{theorem}

\begin{theorem}[\bfseries SGD]\label{thm disc SGD 2NN}\ Let $\Theta{(t)}$ are trained by SGD \eqref{disc SGD} with random initialization \eqref{random initialization}. The learning rate $\eta_t$ is selected in the same way as Theorem \ref{thm disc GD 2NN}. And the discrete cumulative loss ${\rm\sf{CL}}(T)$ at $T$ is defined as \eqref{equ: disc CT}. Then with probability at least $\frac{\rho}{1+\rho}-\delta$, we have:
\[
\mathcal{E}_{\rm gen}(\Theta(T))\lesssim
\frac{C_{L,d}}{m^p\sqrt{n}} \Bigg(
\mathcal{O}\Big((1+\rho)\kappa^2\Big)+(1+\rho)\mathbb{E}_T\Big[{\rm\sf{CL}}(T)\Big]\Bigg)^\frac{L+1}{2}+\sqrt{\frac{\log(1/\delta)}{n}},
\]
where $\mathbb{E}_T=\mathbb{E}_{\gamma^1,\cdots,\gamma^{T-1}}$ and $C_{L,d}$ is defined in Lemma \ref{Estimation of Rademacher Complexity}.	
\end{theorem} 

Theorem \ref{continuous GD thm}, 
\ref{thm disc GD 2NN} and 
\ref{thm disc SGD 2NN} provide a novel class of generalization error bounds depending on the cumulative loss. Our bounds are non-vacuous, and grow slowly during most of the time (see numerical results in Section~\ref{section: exp}). Moreover, our bounds have wide range of applicability on models and algorithms, which only need the network function to be bounded without any other assumptions like $L-$smooth.

\section{Proof Sketch}\label{section: proof sketch}
In this section, we will discuss the proof sketch of our generalization theorems. 

\subsection{Outline}
First, we dissect two important properties for deep neural network that we will use extensively---the homogeneity property (Prop \ref{multiplicative property}) and the parameter-based upper bounds for networks and gradients (Prop \ref{property: high order upper bound}). 
Second, we derive new width-independent Rademacher complexity estimate (Lemma~\ref{Estimation of Rademacher Complexity}) based on the parameter norm for deep CNNs.
Finally, combining the estimate of Rademacher complexity and the two properties for deep networks, we dynamically control parameter norm along the trajectory of optimization algorithms. For GF, we obtain our generalization bound depending on the cumulative loss by direct application of the homogeneity property. For GD, in addition to the analysis for GF, we need to bound a sum of quadratic terms using a more fine grained analysis with the two properties. To extend the analysis to SGD, which is the optimizer we are most interested in, we apply the homogeneity property on individual data and control the parameter norm in high probability sense. 
Please refer to all detailed proof in appendix \ref{section: proof intrinsic properties}, \ref{section: proof Rademacher complexity}, \ref{section: proof GF}, \ref{section: proof GD} and \ref{section: proof SGD}.

\subsection{Import Properties for deep neural networks}\label{section structure and complexity}
In this section, we introduce two important results for deep neural networks (with ReLU-like activation functions) that we will use in our proof. 
\begin{proposition}[\bfseries Homogeneity Property]\label{multiplicative property}
	For deep CNNs or FNNs \eqref{deep CNN}, we have:
	\begin{equation}
		\left<\Theta^{(l)},\frac{\partial f(\mathbf{x};\Theta)}{\partial \Theta^{(l)}}\right>=f(\mathbf{x};\Theta),\ l\in[L+1];
		\quad
	\left<\Theta,\nabla_{\Theta}f(\mathbf{x};\Theta)\right>=(L+1)f(\mathbf{x};\Theta)\color{black}.
    \end{equation}
\end{proposition}
Besides, we derive the following parameter norm based upper bounds for neural networks' output and gradients. 
\begin{proposition}[\bfseries network and gradient upper bounds]\label{property: high order upper bound}
   	For deep CNNs or FNNs \eqref{deep CNN}, we have:
	\begin{align}
	&|f(\mathbf{x};\Theta)|\leq\frac{1}{m^p}\prod_{l=1}^{L+1}\left\|\Theta^{(l)}\right\|_2	\leq\frac{1}{m^p(L+1)^{\frac{L+1}{2}}}\left\|\Theta\right\|_2^{L+1};
	\\		
	&\left\|\frac{\partial f(\mathbf{x};\Theta)}{\partial\Theta^{(l)}}\right\|_2
	\leq\frac{1}{m^p}\prod_{i\neq l}\left\|\Theta^{(i)}\right\|_2	\leq\frac{1}{m^p L^{\frac{L}{2}}}\left\|\Theta\right\|_2^{L},\ \forall l\in[L+1].
	\end{align}
\end{proposition}
\begin{remark}\rm
The average pooling layers play a role of regularization on $f$ and $\nabla f$ for model constructions. From the proof of Property \ref{property: high order upper bound}, we can see that the average pooling layers provide a reduction factor ${1}/{\prod_{i=1}^{L_C}\sqrt{s_i}}$ on the upper bounds, where $s_l$ is the size of convolutional kernel and average pooling on the layer $l$. And the property also holds for the active function $\sigma(\cdot)$ which satisfies $|\sigma'(z)|\leq1$.
\end{remark}

The proof sketch of Proposition \ref{multiplicative property} and \ref{property: high order upper bound} is the use of the chain rule and multiplicative structure.  We provide the details in appendix \ref{section: proof intrinsic properties}.

\subsection{\bfseries Norm-based Rademacher Complexity Estimate}\label{section: Rademacher estimation}
The Rademacher complexity is a classical tool to study the generalization of machine learning models. Here we list its definition and the related upper bound for the generalization error.

\begin{definition}[\bfseries Rademacher Complexity]\ \rm If we use $\mathcal{F}$ and $\{\mathbf{x}_i\}_{i=1}^n$ to denote the hypothesis
space and the training data respectively, the Rademacher complexity of $\mathcal{F}$ with
respect to the data is defined by
${\rm Rad}_n(\mathcal{F})=\frac{1}{n}\mathbb{E}_{\sigma_1,\cdots,\sigma_n}\Big[\sup\limits_{f\in\mathcal{F}}\sum\limits_{i=1}^n \sigma_i f(\mathbf{x}_i)\Big]$,
where $\{\sigma_i\}_{i=1}^n$are i.i.d Rademacher random variables with $\mathbb{P}(\sigma=1)=\mathbb{P}(\sigma=-1)=\frac{1}{2}$.
\end{definition}
\begin{lemma}[\citep{mohri2018foundations}]\label{lemma: gen and Rad}\ Assume that the loss function $\ell(\cdot, y)$ is $\rho$-Lipschitz
continuous and bounded in $[0,C]$. For any $\delta\in(0,1)$, with probability at least $1-\delta$ over the
random sampling of the training data, the following generalization bound hold for any $f\in\mathcal{F}:\ 
\mathcal{L}_{\mu}(f)\leq\mathcal{L}_{n} (f) + 2 \rho {\rm Rad}_n(\mathcal{F}) + 3C\sqrt{\frac{\log(2/\delta)}{2n}}$.
\end{lemma}

Non-vacuous norm-dependent estimation of Rademacher complexity is crucial in our proof process. Theorem 1 of \citep{golowich2018size} has proposed effective estimation independent of the number of training parameters for FNNs. 
However, if we treat CNNs as a special case of FNNs and use theorem 1 of \citep{golowich2018size} directly, we can only get the estimation depending on the width of each convolutional layer $\prod_{l=1}^{L_C}\sqrt{m_l}$. To get rid of this width dependent term for CNNs, we conduct finer analysis by a peeling technique on the $(1,\infty)$-norm of convolution parameters and $F$-norm of fully-connected parameters. The following lemma gives the estimate of Rademacher complexity combining CNNs and FNNs. We provide the details in appendix \ref{section: proof Rademacher complexity}.
\begin{lemma}[\textbf {Rademacher Complexity Estimate}]\label{Estimation of Rademacher Complexity}\ Consider deep CNNs or FNNs \eqref{deep CNN}. If we let $\mathcal{F}_{\mathbf{Q}}=\big\{f(\cdot,\Theta):\left\|\Theta^{(l)}\right\|_2\leq Q_l,\forall l\in[L+1]\big\}$ where $\mathbf{Q}=(Q_1,\cdots,Q_{L+1})^\top$, then we have the Rademacher complexity estimation
${\rm Rad}_n(\mathcal{F}_{\mathbf{Q}})\leq
\frac{C_{L,d}}{m^p\sqrt{n}}\prod_{l=1}^{L+1}Q_l$,
where $C_{L,d}=2\sqrt{(L+2+\log d)d}$ for deep CNNs and $C_{L,d}=\sqrt{2(L+1)\log 2}+1$ for deep FNNs.
\end{lemma}

\subsection{Continuous Time Analysis}\label{subsection: proof sketch: GF}
As a warm-up for discrete case, we give our proof sketch for GF (\ref{conti GD}) case.
As the application of the homogeneity property \ref{multiplicative property}, we can build the relationship between norm dynamics of each layer $l\in[L+1]$ and and the training term:
\begin{align*}
&\frac{\mathrm{d} \left\|\Theta^{(l)}(t)\right\|^2}{\mathrm{d} t}=-2\left<\Theta^{(l)}(t),\frac{\partial \mathcal{L}_n(\Theta(t))}{\partial \Theta^{(l)}}\right>
=-\frac{2}{n}\sum_{i=1}^n\Big(f(\mathbf{x}_i;\Theta(t))-y_i\Big)\left<\Theta^{(l)}(t),\frac{\partial f(\mathbf{x}_i;\Theta(t))}{\partial \Theta^{(l)}}\right>	\\=&-\frac{2}{n}\sum_{i=1}^n\Big(f(\mathbf{x}_i;\Theta{(t)})-y_i\Big)f(\mathbf{x}_i;\Theta{(t)})=-4\mathcal{L}_n(\Theta{(t)})-\frac{2}{n}\sum_{i=1}^n\Big(f(\mathbf{x}_i;\Theta{(t)})-y_i\Big)y_i
\\\leq&
-4\mathcal{L}_n(\Theta{(t)})+\frac{2}{n}\sqrt{\sum_{i=1}^n\Big(f(\mathbf{x}_i;\Theta{(t)})-y_i\Big)^2}\sqrt{\sum_{i=1}^n y_i^2}
\leq-4\mathcal{L}_n(\Theta{(t)})+2C_y\sqrt{2\mathcal{L}_n(\Theta{(t)})}.
\end{align*}
Integrating time, $\left\|\Theta^{(l)}(T)\right\|_2^2$ can be bounded by the initial scale $\kappa^2$ and the continuous cumulative loss ${\rm \sf CL}(T)=\int_{0}^T2\sqrt{2\mathcal{L}_n(\Theta{(t)})}\big(C_y-\sqrt{2\mathcal{L}_n(\Theta{(t)})}\big)\mathrm{d}t$ with high probability. We provide the details in appendix \ref{section: proof GF}.

\subsection{Fine-grained Discrete Time Analysis}
\textbf{GD.} First, we can decompose the time dependent parameter norm in each layer into three parts:
\[
\left\|\Theta^{(l)}(T+1)\right\|_2^2\notag
=
\underbrace{\left\|\Theta^{(l)}(0)\right\|_2^2}_{\rm I}
+
\underbrace{\sum\limits_{t=0}^T2\eta_t\left<\Theta^{(l)}(t),\frac{\partial \mathcal{L}_n(\Theta(t))}{\partial \Theta^{(l)}}\right>}_{\rm II}
+\underbrace{\sum\limits_{t=0}^{T}\eta_t^2\left\|\frac{\partial \mathcal{L}_n(\Theta(t))}{\partial \Theta^{(l)}}\right\|_2^2}_{\rm III}.\]
Term ${\rm I}=\left\|\Theta^{(l)}(0)\right\|_2^2$ can be bounded by the initial scale $\mathcal{O}(\kappa^2)$ with high probability. Term ${\rm II}$ can be bounded by discrete cumulative loss ${\rm\sf CL}(T)$ using the homogeneity property \ref{multiplicative property} like the GF case:
\[{\rm II}
    =-\sum_{t=0}^T2\eta_t\left<\Theta^{(l)}(t),\frac{\partial \mathcal{L}_n(\Theta(t))}{\partial \Theta^{(l)}}\right>
    =-\sum_{t=0}^T\eta_t\frac{2}{n}\sum_{i=1}^n(f(\mathbf{x}_i;\Theta{(t)})-y_i)f(\mathbf{x}_i;\Theta{(t)})
    \leq{\rm\sf CL}(T).\]
And we need more fine-grained analysis than GF due to the extra term ${\rm III}
    \leq\sum_{t=0}^{T}\eta_t^2(C_f+C_y)^2\sup_{\mathbf{x}}\left\|\frac{\partial f(\mathbf{x},\Theta(t))}{\partial \Theta^{(l)}}\right\|_2^2$ in discrete time.
From the gradient upper bound property \ref{property: high order upper bound}, we can see that the growth rate of gradients is close to $\mathcal{O}\big(\left\|\Theta^{(l)}(t)\right\|_2^{2L}\big)$, which may keep growing. So it seems that term III 
may not converge.
However, from the important properties \ref{multiplicative property}, \ref{property: high order upper bound} and estimation of term II, if we choose the popular learning rate $\eta_t=\eta/\lceil\frac{t+1}{T_0}\rceil^\alpha$,  $\alpha\in(\frac{L+1}{L+2},1]$, under proper selection of $\eta$, term III can be bounded by the tiny constant $\mathcal{O}(\kappa^2)$. Moreover, there exist $c>0$, $\epsilon\in(0,1)$ s.t.
\[{\rm III}\leq\mathcal{O}\big(\left\|\Theta^{(l)}(0)\right\|_2^2\big)-\frac{c}{(T+1)^{\epsilon}}\leq\mathcal{O}\big(\left\|\Theta^{(l)}(0)\right\|_2^2\big).\]

So with high probability over random initialization,
 GD norm dynamics of each layer can be bounded by $
\left\|\Theta^{(l)}(T+1)\right\|_2^2\leq\mathcal{O}(\kappa^2)+{\rm \sf CL}(T)$.
We provide the details in appendix \ref{section: proof GD}.

\textbf{SGD.} Let $\mathbb{E}_T:=\mathbb{E}_{\gamma^0,\cdots,\gamma^T}=\mathbb{E}_{(\gamma_1^0,\cdots,\gamma_B^0),\cdots,(\gamma_1^T,\cdots,\gamma_B^T)}$ for $T\geq 0$ and $\mathbb{E}_{-1}:=id$. We will control 
\[
\mathbb{E}_T\Big[\left\|\Theta^{(l)}(T+1)\right\|_2^2\Big]
=
\underbrace{\left\|\Theta^{(l)}(0)\right\|_2^2}_{\rm I}
+\underbrace{\sum\limits_{t=0}^T2\eta_t\mathbb{E}_t\left<\Theta^{(l)}(t),\frac{\partial \mathcal{L}_n(\Theta(t))}{\partial \Theta^{(l)}}\right>}_{\rm II}
+\underbrace{\sum\limits_{t=0}^{T}\eta_t^2\mathbb{E}_t\left\|\frac{\partial \mathcal{L}_n(\Theta(t))}{\partial \Theta^{(l)}}\right\|_2^2}_{\rm III}.
\]

Term I can be bounded by $\mathcal{O}(\kappa^2)$ as GD case. For each term in II, we can bound it by the homogeneity property on individual data and the conditional expectation formula:
\begin{align*}
&2\eta_t\mathbb{E}_t\left<\Theta^{(l)}(t),\frac{\partial \mathcal{L}_n(\Theta(t))}{\partial \Theta^{(l)}}\right>
=-2\eta_t\mathbb{E}_{t-1}\Bigg[\mathbb{E}_{\gamma^t}\Big[\frac{1}{B}\sum_{i=1}^B\Big(f(\mathbf{x}_{\gamma_i^t};\Theta{(t)}) -y_{\gamma_i^t}\Big)f(\mathbf{x}_{\gamma_i^t};\Theta{(t)})\Big|\gamma^0,\cdots,\gamma^{t-1}\Big]\Bigg]
\\
&=-2\eta_t\mathbb{E}_{t-1}\Big[\frac{1}{n}\sum_{i=1}^n\Big(f(\mathbf{x}_i;\Theta{(t)})  -y_i\Big)f(\mathbf{x}_i;\Theta{(t)})\Big]
\leq
2\eta_t\mathbb{E}_{t-1}\Big[\sqrt{2\mathcal{L}_n(\Theta(t))}(C_y-\sqrt{2\mathcal{L}_n(\Theta(t))})\Big].
\end{align*}

Then for term III, the analysis is similar to GD. We provide the details in appendix \ref{section: proof SGD}.

\section{Comparison with Uniform Stability Bounds}\label{section:versus}
In this section, we give a detailed comparison of our results with the uniform stability based bounds. We focus on the setting and condition required for the results, especially in the non-$L$-smooth scenario. First, we point out that the SGLD is the only algorithm that can be treated by the uniform stability theory in the non-$L$-smooth case.
Recalling the continuous form and discrete form of SGLD are:
\begin{gather}
\mathrm{d}\Theta{(t)}=-\nabla \mathcal{L}(\Theta{(t)})\mathrm{d}t+\sqrt{\frac{2}{\beta}}\mathcal{B}(t),
\label{conti SGLD}
\\
\Theta{(t+1)}=\Theta{(t)}-\eta_t\nabla\mathcal{L}(\Theta{(t)})+\sqrt{\frac{2\eta_t}{\beta}}\mathcal{N}(\mathbf{0},\mathbf{I}).
\label{disc SGLD}
\end{gather}
A notable uniform stability based result for the generalization performance of SGLD is given in~\citep{mou2018generalization}:

\begin{lemma}[\citep{mou2018generalization}]\label{uniform}\ Under boundedness assumption and Lipschitz assumption, 
let $\Theta{(t)}$ be trained by discrete SGLD $($algorithm $(\ref{disc SGLD}))$. Then we have the expectation of generalization error $\mathbb{E}\Big[\mathcal{E}_{\rm{gen}}(T)\Big]\leq MLip\sqrt{\frac{\beta}{8n}\sum_{t=0}^{T-1}\eta_t}$.
Similarly, let $\Theta{(t)}$ be trained by continuous SGLD $($algorithm $(\ref{conti SGLD}))$, then we have
$\mathbb{E}\Big[\mathcal{E}_{\rm{gen}}(T)\Big]\leq \frac{M Lip\sqrt{\beta T}}{\sqrt{2}n}$.
\end{lemma}

It is clear that GF and GD are the limits of SGLD while $\beta\to\infty$ in algorithm (\ref{conti SGLD}) and (\ref{disc SGLD}). Hence the generalization bounds for GF and GD based on Lemma~\ref{uniform} are vacuous ($+\infty$). And for SGD, so far there is no generalization result based on uniform stability for SGD beyond SGLD.

Now we can compare the effectiveness between our bounds (Theorem $\ref{continuous GD thm}$,
$\ref{thm disc GD 2NN}$,
$\ref{thm disc SGD 2NN}$) and uniform stability bounds (Lemma $\ref{uniform}$) in the non-$L$-smooth deep CNN and FNN scenario. Results are given in Table~\ref{sample-table}.
The table shows that our bounds can be used to analyze GD and SGD beyond SGLD. But per our knowledge, in the non-
L-smooth scenario, the uniform stability theory has only been applied to analyze GD with isotropic noise (SGLD) rather than GD or SGD.
Besides, our bounds only need the boundedness assumption without any other assumptions.

\section{Experiments}\label{section: exp}
\subsection{Experiments about our bounds}
\textbf{Regression and Classification.}
We test our generalization bounds on a function regression problem and classification problems with quadratic loss. For the function regression problem, we train a 3-depth FNN by SGD (batch=2000) on the target function $f^{*}(\mathbf{x})=(x_1+x_2^2+\sin(\pi x_3))/(1.25+\pi^2/4)$, where $\mathbf{x}\in[-1/\sqrt{3},1/\sqrt{3}]^3$. For the classification problem, we train a 4-depth FNN on the MNIST dataset \citep{lecun1998gradient} with label$=0,1$ 
by SGD (batch=2000). Figures \ref{figure: function regression bound} and \ref{figure: MNIST bound} show our bounds in the two experiments. It is clear that our bounds grow slowly, hold steady during most of the time, and keep close to initial bounds, especially for the function regression problem (Figure \ref{figure: function regression bound}). 

\textbf{Effect of Hyperparameters.}
Some generalization bounds are sensitive to hyperparameters of models and algorithms such as width, depth and learning rate \citep{zhou2018understanding, lin2019generalization}. We study the change of our generalization bounds under different network width and learning rate. Results shown in Figure~\ref{figure: bound-width} and \ref{figure: bound-lr} show that our bounds are not sensitive to the change of these hyperparameters.

\begin{figure}[ht]
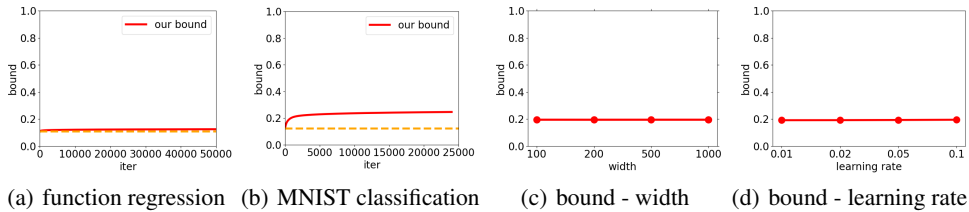

\begin{center}
\subfigure[function regression]{
\label{figure: function regression bound}
\includegraphics[width=0.22\textwidth]{figure/toy_model_bound.pdf}}
\subfigure[MNIST classification]{
\label{figure: MNIST bound}
\includegraphics[width=0.22\textwidth]{figure/MNIST_bound.pdf}}
\subfigure[bound - width]{
\label{figure: bound-width}
\includegraphics[width=0.22\textwidth]{figure/bound-width.pdf}}
\subfigure[bound - learning rate]{
\label{figure: bound-lr}
\includegraphics[width=0.22\textwidth]{figure/bound-lr.pdf}}
\end{center}
\caption{Fig \ref{figure: function regression bound}: Our generalization bound of function regression during training; Fig \ref{figure: MNIST bound}:
Our generalization bound of MNIST classification during training. Fig \ref{figure: bound-width}:  Our generalization bounds under different widths; Fig \ref{figure: bound-lr}:
Our generalization bounds under different learning rates.}
\label{MNIST exp}
\end{figure}

\subsection{Large-scale Experiments about CL}
In order to further understand the relationship between the generalization ability and the main component ${\rm\sf{CL}}$ (\ref{equ: disc CT}) of our generalization bounds, we conduct \textit{three} groups of large-scale experiments with different proportion of label noise, different number of class and different network sizes. We consider to classify the CIFAR-10 dataset~\citep{krizhevsky2009learning} with VGG networks~\citep{simonyan2014very}. All models (without batch normalization) are trained by SGD (batch size=100, learning rate=0.1) until for $10^6$ iterations, and the results are shown in Table \ref{table: large-scale} below.
\begin{table}[ht]
\caption{These three subtables show the results of the following experiments. 
{(I)} VGG-16; Cifar-10 (label=0,1) with different proportion of label noise $0$\%, $20$\%, $50$\%, $80$\%, $100$\%.
{(II)} VGG-16; subset of Cifar-10 with different number of class: $2$, $5$, $8$, $10$.
{(III)} Different network sizes VGG-11, VGG-13, VGG-16, VGG-19; on a subset of Cifar-10 (label=0, 1).}
\label{table: large-scale}
\begin{center}
\begin{tabular}{c|c|c|c|c|c}
      \hline
     label noise & 0\% & 20\% & 50\% & 80\% & 100\% \\  \hline
     ${\rm\sf{CL}}$ & $\mathbf{6.60}$ & $21.20$ & $37.25$ & $50.57$ & $62.65$ \\
	\hline			
\end{tabular}
\begin{tabular}{c|c|c|c|c}
      \hline    
     data class & 2 & 5 & 8 & 10  \\  \hline
     ${\rm\sf{CL}}$ & $\mathbf{6.60}$ & $19.75$ & $25.24$ & $26.08$ \\
	\hline			
\end{tabular}
\begin{tabular}{c|c|c|c|c}
      \hline
     network & VGG-11 & VGG-13 & VGG-16 & VGG-19 \\  
     \hline
     ${\rm\sf{CL}}$ & $15.51$ & $9.52$ & $\mathbf{6.60}$ & $7.88$ \\
	\hline	
\end{tabular}
\end{center}
\end{table}

From the results in Table \ref{table: large-scale}, data complexity introduced by label noise and number of classes makes the ${\rm\sf{CL}}$ larger, which shows that it is more difficult to learn generalizable representations for more complicated dataset. On the other hand, however, the network size has a negative correlation with the ${\rm\sf{CL}}$, which reflect the better generalization ability of larger networks.

More experiment details are provided in appendix \ref{section: appendix experiment}.

\section{Conclusion}
\label{section: conclusion}
In this paper, we derive novel algorithm-dependent generalization bounds for deep neural networks under weak conditions without smoothness assumption of the loss function. In the analysis, we combine parameter norm based bounds with an analysis of the optimization trajectory, and make use of two special properties of deep neural networks---the homogeneity property and the value/gradient upper bounds. Our generalization bounds explicitly depend on the training process, and work for a wide range of network architectures including general deep CNNs and FNNs. The bounds also apply to popular optimization algorithms such as GD and SGD.
As a comparison, uniform stability can only treat SGLD rather than GD/SGD in the non-$L$-smooth scenario.
Numerical experiments show that our bounds are non-vacuous and robust with the change of optimizer and network hyperparameters, such as the learning rate and width. Our analysis can also be extended to other power-type loss (see Appendix \ref{section: proof loss extension}). The analysis for more general loss functions, such as exponential-type loss and cross-entropy loss, may be a topic of future work.

\newpage

\appendix

\section{\bfseries Proof Details of Section \ref{section structure and complexity}}\label{section: proof intrinsic properties}
Recalling the deep CNN and FNN model (\ref{deep CNN}). For deep CNN case, the input $\mathbf{x}$ may be vectors (in $\mathbb{R}^d$) or matrixes (in $\mathbb{R}^{\sqrt{d}\times\sqrt{d}}$) . For matrix case, it is the same as the vector case by resizing the input and covolution kernels to vectors in the proof, and the same results can also be derived because the multiplicative structure of neural networks also holds. So without loss of generality, we only need to consider the vector case. And the specific form of deep CNN (\ref{deep CNN}) can be written as:
\[
\begin{aligned}
{\rm Output:}\ \ 	f(\mathbf{x};\Theta)&=\frac{1}{m^p}\sum_{k=1}^{m}a_k{z}_k^{(L)}, \\
{\rm FC:}\ \ \ \ \ \ \ \ \ 	\mathbf{z}^{(l)}&=\sigma(\mathbf{A}^{(l)^{\top}}\mathbf{z}^{(l-1)}),\ l\in[L]-[L_C], \\
{\rm Pool:}\ \ \ \ \ \ \ \ \ 	z_k^{(l)} &= \frac{1}{s_l}\sum\limits_{j=1}^{s_l}y_{(k-1)s_l+j}^{(l)} ,\ k\in[m_l],\ l\in[L_C] ,\\
{\rm Conv:}\ \ \ \ \ \ \ \ \ 	y_k^{(l)}&=\sigma(\mathbf{w}^{(l)^\top}\mathbf{z}_{[k:k+s_l-1]}^{(l-1)}),\ k\in[m_ls_l],\ l\in[L_C], \\
{\rm Input:}\ \ \ \ \ \ \ \  \mathbf{z}^{(0)}&=\mathbf{x}\in\mathbb{R}^d.
    \end{aligned}
\]
	
\begin{proof}[\bfseries Proof of Property \ref{multiplicative property}]\ 
\\ 
For the model \eqref{deep CNN} parameterized by $\Theta=({\Theta^{(1)})}^\top,\cdots,{\Theta^{(L+1)})}^\top)^\top)$, it holds that: for any layer $l\in[L+1]$,
\begin{align*}
    f\left(\mathbf{x};({\Theta^{(1)}}^\top,\cdots,c{\Theta^{(l)}}^\top,\cdots,{\Theta^{(L+1)}}^\top)^\top\right)=c f\left(\mathbf{x};\Theta\right),\forall c>0.
\end{align*}

Deriving the above formula with respect to $c$, and substituting $c=1$, we obtain:

\[
\left<\Theta^{(l)},\frac{\partial f}{\partial \Theta^{(l)}}\right>
=f(\mathbf{x};\Theta),\ \forall l\in [L+1].
\]
Then we have
\[
\left<\Theta, \nabla_{\Theta} f \right>
=(L+1)f(\mathbf{x};\Theta).
\]
\end{proof}
\begin{proof}[\bfseries Proof of Property \ref{property: high order upper bound}]\ \\
\noindent
We only need to prove the property for deep CNNs. As the beginning of the proof, we use some notations to simplify the forms.

For fully-connected layers, $\mathbf{A}^{(l)}=(\mathbf{A}_1^{(l)},\cdots,\mathbf{A}_{m_l}^{(l)})\in\mathbb{R}^{m_{l-1}\times m_l}$, $\mathbf{A}^{(L+1)}=(a_1,\cdots,a_m)^\top\in\mathbb{R}^{m}$ and $\Theta^{(l)}={\rm vec}(\mathbf{A}^{(l)})\ (l\in[L]-[L_C])$. Then the fully-connected layers can be writen as:
\[
\begin{gathered}
f(\mathbf{x};\Theta)=\frac{1}{m^p}{\mathbf{A}^{(L+1)}}^{\top}\mathbf{z}^{(L)}
\\
\mathbf{z}^{(l)} = \sigma\Big({\mathbf{A}^{(l)}}^{\top}\mathbf{z}^{(l-1)}\Big),\ l\in[L]-[L_C+1]
\end{gathered}
\]

For convolution layer $l\in[L_C]$,
we expand $\mathbf{w}^{(l)}\in\mathbb{R}^{s_l}$ to $\mathbf{\Tilde{w}}_{(k)}^{(l)}\in\mathbb{R}^{m_{l-1}}$ by:
\[
\Tilde{w}_{(k),j}^{l}=\begin{cases}w_{j-k}^{(l)},\ {\rm for}\  j:\ k\leq j\leq k+s_l-1\\0,\ {\rm otherwise}\end{cases}, k\in[m_ls_l],
\]
then we have $y_k^{(l)}=\sigma(\mathbf{\Tilde{w}}_{(k)}^{(l)^{\top}}\mathbf{z}^{(l-1)})$. If we denote:
\[
\begin{gathered}
 \mathbf{\Tilde{W}}^{(l)} = (\mathbf{\Tilde{w}}_{(1)}^{(l)},\cdots,\mathbf{\Tilde{w}}_{(m_ls_l)}^{(l)})\in\mathbb{R}^{m_{l-1}\times m_ls_l},
 \\
 \mathbf{P}^{(l)}=(p_{i j}^{(l)})_{m_l\times m_{l}s_l}, \ \ \ \ p_{i j}^{(l)}=\begin{cases}
\frac{1}{s_l},\ (i-1)s_l+1\leq j \leq is_l\\
0,\ {\rm otherwise}
\end{cases}
\end{gathered}
\]
the convolution layers of $l\in[L_C]$ can be written as:
\[
\begin{gathered}
    \mathbf{z}^{(l)}=\mathbf{P}^{(l)}\mathbf{y}^{(l)},\\
    \mathbf{y}^{(l)} =\sigma\Big( {\mathbf{\Tilde{W}}^{(l)^\top}}\mathbf{z}^{(l-1)}\Big).
\end{gathered}
\]
So the deep CNN can be writen as:
\[
 f(\mathbf{x};\Theta)=\frac{1}{m^p}{\mathbf{A}^{(L+1)^\top}}\sigma\Big({\mathbf{A}^{(L)}}^{\top}\sigma\Big(\cdots \mathbf{P}^{(L_C)}\sigma\Big({\mathbf{\Tilde{W}}^{(L_C-1)^\top}}\mathbf{P}^{(L_C-1)}\sigma\Big(\cdots{\mathbf{\Tilde{W}}^{(1)^\top}}\mathbf{x}\Big) \Big)\Big)\Big)\]

\noindent \textbf{(I) The bound of $f$.}
\[
\begin{aligned}
 |f(\mathbf{x};\Theta)|&=
 \frac{1}{m^p}\Big|{\mathbf{A}^{(L+1)}}^{\top}\sigma\Big({\mathbf{A}^{(L)}}^{\top}\sigma\Big(\cdots \mathbf{P}^{(L_C)}\sigma\Big({\mathbf{\Tilde{W}}^{(L_C)^\top}}\mathbf{P}^{(L_C-1)}\sigma\Big(\cdots{\mathbf{\Tilde{W}}^{(1)^\top}}\mathbf{x}\Big) \Big)\Big)\Big)\Big|
 \\&
 \leq
 \frac{1}{m^p}\left\|{\mathbf{A}^{(L+1)}}^{\top}\right\|_2\left\|\sigma\Big({\mathbf{A}^{(L)}}^{\top}\sigma\Big(\cdots \mathbf{P}^{(L_C)}\sigma\Big({\mathbf{\Tilde{W}}^{(L_C)^\top}}\mathbf{P}^{(L_C-1)}\sigma\Big(\cdots{\mathbf{\Tilde{W}}^{(1)^\top}}\mathbf{x}\Big) \Big)\Big)\Big)\right\|_2
 \\&
 \leq
 \frac{1}{m^p}\left\|\Theta^{(L+1)}\right\|_2\left\|{\mathbf{A}^{(L)}}^{\top}\sigma\Big(\cdots \mathbf{P}^{(L_C)}\sigma\Big({\mathbf{\Tilde{W}}^{(L_C)^\top}}\mathbf{P}^{(L_C-1)}\sigma\Big(\cdots{\mathbf{\Tilde{W}}^{(1)^\top}}\mathbf{x}\Big) \Big)\Big)\right\|_2
 \\&
 \leq
 \frac{1}{m^p}\left\|\Theta^{(L+1)}\right\|_2\left\|{\mathbf{A}^{(L)}}^{\top}\right\|_2\left\|\sigma\Big(\cdots \mathbf{P}^{(L_C)}\sigma\Big({\mathbf{\Tilde{W}}^{(L_C)^\top}}\mathbf{P}^{(L_C-1)}\sigma\Big(\cdots{\mathbf{\Tilde{W}}^{(1)^\top}}\mathbf{x}\Big) \Big)\Big)\right\|_2
  \\&
 \leq
 \frac{1}{m^p}\left\|\Theta^{(L+1)}\right\|_2\left\|\Theta^{(L)}\right\|_2\left\|\sigma\Big(\cdots \mathbf{P}^{(L_C)}\sigma\Big({\mathbf{\Tilde{W}}^{(L_C)^\top}}\mathbf{P}^{(L_C-1)}\sigma\Big(\cdots{\mathbf{\Tilde{W}}^{(1)^\top}}\mathbf{x}\Big) \Big)\Big)\right\|_2
  \\&
 \leq
 \frac{1}{m^p}\left\|\Theta^{(L+1)}\right\|_2\cdots\left\|\Theta^{(L_C+1)}\right\|_2\left\| \mathbf{P}^{(L_C)}\sigma\Big({\mathbf{\Tilde{W}}^{(L_C)^\top}}\mathbf{P}^{(L_C-1)}\sigma\Big(\cdots{\mathbf{\Tilde{W}}^{(1)^\top}}\mathbf{x}\Big) \Big)\right\|_2 ,
 \\&
 \leq
 \frac{1}{m^p}\left\|\Theta^{(L+1)}\right\|_2\cdots\left\|\Theta^{(L_C+1)}\right\|_2\left\| \mathbf{P}^{(L_C)}\right\|_2\left\|\sigma\Big({\mathbf{\Tilde{W}}^{(L_C)^\top}}\mathbf{P}^{(L_C-1)}\sigma\Big(\cdots{\mathbf{\Tilde{W}}^{(1)^\top}}\mathbf{x}\Big) \Big)\right\|_2,
\end{aligned}
\]
\[
{\rm where\ }
\left\| \mathbf{P}^{(l)}\right\|_2=\lambda_{\max}^{1/2}\Big(\mathbf{P}^{(l)}{\mathbf{P}^{(l)^\top}}\Big)
=\lambda_{\max}^{1/2}\Big(\frac{1}{s_{l}}\mathbf{I}\Big)=\frac{1}{\sqrt{s_{l}}},\ l\in[L_C],
\]
\[
\begin{aligned}
{\rm and\ }
&\left\|{\mathbf{\Tilde{W}}^{(l)^\top}}\mathbf{z}^{l-1}\right\|_2
\\=&
\left\|\Big(w_1^{(l)}z_1^{(l-1)}+\cdots+w_{s_l}^{(l)}z_{s_l}^{(l-1)},w_1^{(l)}z_2^{(l-1)}+\cdots+w_{s_l}^{(l)}z_{s_l+1}^{(l-1)},\cdots,
 w_1^{(l)}z_{m_{l-1}-s_l+1}^{(l-1)}+\cdots+w_{s_l}^{(l)}z_{m_{l-1}}^{(l-1)}
 \Big)^{\top}\right\|_2
 \\\leq&\sqrt{
 \left\|\mathbf{w}^{(l)}\right\|_2^2
 \Big(({z_1^{(l-1)}}^2+\cdots+{z_{s_l}^{(l-1)}}^2)+({z_2^{(l-1)}}^2+\cdots+{z_{s_l+1}^{(l-1)}}^2)+\cdots+(
 {z_{m_{l-1}-s_l+1}^{(l-1)}}^2+\cdots+{z_{m_{l-1}}^{(l-1)}}^2
 )\Big)}
 \\=&\sqrt{s_l}\left\|\mathbf{w}^{(l)}\right\|_2\left\|\mathbf{z}^{(l-1)}\right\|_2\ l\in[L_C].
\end{aligned}
\]
So we have the bound:
\[
\begin{aligned}
  &|f(\mathbf{x};\Theta)|
  \\\leq& \frac{1}{m^p}\left\|\Theta^{(L+1)}\right\|_2\cdots\left\|\Theta^{(L_C+1)}\right\|_2\left\| \mathbf{P}^{(L_C)}\right\|_2\left\|\sigma\Big({\mathbf{\Tilde{W}}^{(L_C-1)^\top}}\mathbf{P}^{(L_C-1)}\sigma\Big(\cdots{\mathbf{\Tilde{W}}^{(1)^\top}}\mathbf{x}\Big) \Big)\right\|_2
  \\\leq& \frac{1}{m^p}\left\|\Theta^{(L+1)}\right\|_2\cdots\left\|\Theta^{(L_C+1)}\right\|_2\frac{1}{\sqrt{s_{L_C}}}\left\|{\mathbf{\Tilde{W}}^{(L_C-1)^\top}}\mathbf{P}^{(L_C-1)}\sigma\Big(\cdots{\mathbf{\Tilde{W}}^{(1)^\top}}\mathbf{x}\Big)\right\|_2
\\\leq& \frac{1}{m^p}\left\|\Theta^{(L+1)}\right\|_2\cdots\left\|\Theta^{(L_C+1)}\right\|_2\frac{1}{\sqrt{s_{L_C}}}\sqrt{s_{L_C}}\left\|\mathbf{w}^{(L_C)}\right\|_2\left\|\mathbf{P}^{(L_C-1)}\sigma\Big(\cdots{\mathbf{\Tilde{W}}^{(1)^\top}}\mathbf{x}\Big)\right\|_2
\\=&\frac{1}{m^p}\left\|\Theta^{(L+1)}\right\|_2\cdots\left\|\Theta^{(L_C+1)}\right\|_2\left\|\Theta^{(L_C)}\right\|_2\left\|\mathbf{P}^{(L_C-1)}\sigma\Big(\cdots{\mathbf{\Tilde{W}}^{(1)^\top}}\mathbf{x}\Big)\right\|_2
\\\leq&\frac{1}{m^p}\left\|\Theta^{(L+1)}\right\|_2\cdots\left\|\Theta^{(L_C+1)}\right\|_2\left\|\Theta^{(L_C)}\right\|_2\cdots\left\|\Theta^{(1)}\right\|_2\left\|\mathbf{x}\right\|_2
\\\leq&
\frac{1}{m^p(L+1)^{\frac{L+1}{2}}}\left\|\mathbf{x}\right\|_2\left\|\Theta\right\|_2^{L+1}.
\end{aligned}
\]

\noindent \textbf{(II) The bound of $\frac{\partial f}{\partial\Theta^{(l)}}$.}

\noindent$\bullet$ For $l=L+1$, we have:
\[
 \left\|\frac{\partial f(\mathbf{x};\Theta)}{\partial \Theta^{(L+1)}}\right\|_2=\frac{1}{m^p}\left\|\mathbf{z}^{(L)}\right\|_2
 \leq\frac{1}{m^p}\left\|\Theta^{(L)}\right\|_2\cdots\left\|\Theta^{(1)}\right\|_2\left\|\mathbf{x}\right\|_2
\leq\frac{1}{m^p L^{\frac{L}{2}}}\left\|\mathbf{x}\right\|_2\left\|\Theta\right\|_2^L.
\]

\noindent$\bullet$ For $l=L$,
\[
 \Big|\frac{\partial f(\mathbf{x};\Theta)}{\partial {W}_{j i}^{(L)}}\Big|
 =\Big|\left<\frac{\partial f(\mathbf{x};\Theta)}{\partial {\mathbf{z}}^{(L)}},\frac{\partial {\mathbf{z}}^{(L)}}{\partial {W}_{j i}^{(L)}}\right>\Big|
 =\Big|\frac{\partial f(\mathbf{x};\Theta)}{\partial {z}_i^{(L)}}\frac{\partial z_i^{(L)}}{\partial {W}_{j i}^{(L)}}\Big|
 \leq\frac{1}{m^p}|a_i|\Big|z_j^{(L-1)}\Big|,
 \]
 \[
 \begin{aligned}
  &\left\|\frac{\partial f(\mathbf{x};\Theta)}{\partial \Theta^{(L)}}\right\|_2=\left\|\frac{\partial f(\mathbf{x};\Theta)}{\partial {\mathbf{W}}^{(L)}}\right\|_F=\Bigg(\sum_{i,j}\Big(\frac{\partial f(\mathbf{x};\Theta)}{\partial {W}_{j i}^{(L)}}\Big)^2\Bigg)^{\frac{1}{2}}\leq\frac{1}{m^p}\Big(\sum_{i,j}a_i^2{z_j^{(L-1)}}^2\Big)^{\frac{1}{2}}
 \\
 \leq&\frac{1}{m^p}\left\|\Theta^{(L+1)}\right\|_2\left\|\mathbf{z}^{(L-1)}\right\|_2=\frac{1}{m^p}\left\|\Theta^{(L+1)}\right\|_2\left\|\Theta^{(L-1)}\right\|_2\cdots\left\|\Theta^{(1)}\right\|_2\left\|\mathbf{x}\right\|_2
 \\\leq&\frac{1}{m^p L^{\frac{L}{2}}}\left\|\mathbf{x}\right\|_2\left\|\Theta\right\|_2^L.
 \end{aligned}
\]

\noindent$\bullet$ For $l\in[L-1]-[L_C]$, it is more complicated. In order to avoid vague amplification, we will use the chain rule carefully by a correct chain $f\to\mathbf{z}^{(L)}\to\mathbf{z}^{(L-1)}\cdots\to\mathbf{z}^{(l+1)}\to\mathbf{W}^{(l)}$, where $\mathbf{z}^{(l+1)}=\sigma\Big({\mathbf{W}^{(l+1)}}^{\top}\sigma\Big({\mathbf{W}^{(l)}}^{\top}\mathbf{z}^{(l-1)}\Big)\Big)$ provides a finer analysis than $\mathbf{z}^{(l)}\to\mathbf{W}^{(l)}$. 
\[
\begin{aligned}
&\Big|\frac{\partial f(\mathbf{x};\Theta)}{\partial {W}_{j i}^{(l)}}\Big|
=\Big|\left<\frac{\partial f(\mathbf{x};\Theta)}{\partial {\mathbf{z}}^{(L)}},\frac{\partial {\mathbf{z}}^{(L)}}{\partial {W}_{j i}^{(l)}}\right>\Big|
\leq\frac{1}{m^p}\left\|\Theta^{(L+1)}\right\|_2\left\|\frac{\partial {\mathbf{z}}^{(L)}}{\partial {W}_{j i}^{(l)}}\right\|_2
\\=&
\frac{1}{m^p}\left\|\Theta^{(L+1)}\right\|_2\left\|\frac{\partial {\mathbf{z}}^{(L)}}{\partial \mathbf{z}^{(L-1)}}\frac{\partial {\mathbf{z}}^{(L-1)}}{\partial {W}_{j i}^{(l)}}\right\|_2
\leq
\frac{1}{m^p}\left\|\Theta^{(L+1)}\right\|_2\left\|\frac{\partial {\mathbf{z}}^{(L)}}{\partial\mathbf{z}^{(L-1)}}\right\|_F\left\|\frac{\partial {\mathbf{z}}^{(L-1)}}{\partial {W}_{j i}^{(l)}}\right\|_2
\\\leq&
\frac{1}{m^p}\left\|\Theta^{(L+1)}\right\|_2
\left\|\mathbf{W}^{(L)}\right\|_F\left\|\frac{\partial {\mathbf{z}}^{(L-1)}}{\partial {W}_{j i}^{(l)}}\right\|_2
=\frac{1}{m^p}\left\|\Theta^{(L+1)}\right\|_2
\left\|\Theta^{(L)}\right\|_2\left\|\frac{\partial {\mathbf{z}}^{(L-1)}}{\partial {W}_{j i}^{(l)}}\right\|_2
\\\leq&\cdots
\leq
\frac{1}{m^p}\left\|\Theta^{(L+1)}\right\|_2\cdots
\left\|\Theta^{(l+2)}\right\|_2\left\|\frac{\partial {\mathbf{z}}^{(l+1)}}{\partial {W}_{j i}^{(l)}}\right\|_2
\\=&
\frac{1}{m^p}\left\|\Theta^{(L+1)}\right\|_2\cdots
\left\|\Theta^{(l+2)}\right\|_2\Bigg(\sum_{k}\Big(\frac{\partial z_k^{(l+1)}}{\partial {W}_{j i}^{(l)}}\Big)^2\Bigg)^{\frac{1}{2}}.
\end{aligned}
\]
\[
\begin{gathered}
 {\rm From\ }
 z_k^{(l+1)}=\sigma\Big({\mathbf{W}_{:,k}^{(l+1)}}^{\top}\mathbf{z}^{(l)}\Big)
 =\sigma\Big(\sum_{s}W_{s k}^{(l+1)}\sigma\Big({\mathbf{W}_{:,s}^{(l)}}^{\top}\mathbf{z}^{(l-1)}\Big)\Big),{\ \rm we\ have:}\\
 \sum_{k}\Big(\frac{\partial z_k^{(l+1)}}{\partial {W}_{j i}^{(l)}}\Big)^2\leq\sum_{k}\Big(W_{i k}^{(l+1)}z_{j}^{(l-1)}\Big)^2=\left\|\mathbf{W}_{i,:}^{(l+1)}\right\|_2^2{z_j^{(l-1)}}^2.
 \end{gathered}
\]
\[
\begin{aligned}
 {\rm Thus,\ }
 &\left\|\frac{\partial f(\mathbf{x};\Theta)}{\partial \Theta^{(l)}}\right\|_2
 =\left\|\frac{\partial f(\mathbf{x};\Theta)}{\partial \mathbf{W}^{(l)}}\right\|_F
 =\Bigg(\sum_{i,j}\Big(\frac{\partial f(\mathbf{x};\Theta)}{\partial {W}_{j i}^{(l)}}\Big)^2\Bigg)^{\frac{1}{2}}
 \\\leq&
 \frac{1}{m^p}\left\|\Theta^{(L+1)}\right\|_2\cdots
\left\|\Theta^{(l+2)}\right\|_2\Bigg(\sum_{i,j}\sum_{k}\Big(\frac{\partial z_k^{(l+1)}}{\partial {W}_{j i}^{(l)}}\Big)^2\Bigg)^{\frac{1}{2}}
\\\leq&
 \frac{1}{m^p}\left\|\Theta^{(L+1)}\right\|_2\cdots
\left\|\Theta^{(l+2)}\right\|_2\Bigg(\sum_{i,j}\left\|\mathbf{W}_{i,:}^{(l+1)}\right\|_2^2{z_j^{(l-1)}}^2\Bigg)^{\frac{1}{2}}
\\=&
\frac{1}{m^p}\left\|\Theta^{(L+1)}\right\|_2\cdots
\left\|\Theta^{(l+2)}\right\|_2\left\|\mathbf{W}^{(l+1)}\right\|_F\left\|\mathbf{z}^{(l-1)}\right\|_2
\\\leq&
\frac{1}{m^p}\left\|\Theta^{(L+1)}\right\|_2\cdots
\left\|\Theta^{(l+1)}\right\|_2\left\|\Theta^{(l-1)}\right\|_2\cdots\left\|\Theta^{(1)}\right\|_2\left\|\mathbf{x}\right\|_2
\leq\frac{1}{m^p L^{\frac{L}{2}}}\left\|\mathbf{x}\right\|_2\left\|\Theta\right\|_2^L.
\end{aligned}\]

$\bullet$ For $l\in[L_C]$, we will use the chain rule carefully by a correct chain $f\to\mathbf{z}^{(L_C)}\to\mathbf{z}^{(L_C-1)}\cdots\to\mathbf{z}^{(l)}\to\mathbf{w}^{(l)}$, where $\mathbf{z}^{(l)}=\mathbf{P}^{(l)}\sigma\Big(\mathbf{\Tilde{W}}^{(l)^{\top}}\mathbf{z}^{(l-1)}\Big)$ provides a fine analysis. 

\[
\begin{aligned}
&\Big|\frac{\partial f(\mathbf{x};\Theta)}{\partial {w}_{ i}^{(l)}}\Big|
=\Big|\left<\frac{\partial f(\mathbf{x};\Theta)}{\partial {\mathbf{z}}^{(L)}},\frac{\partial {\mathbf{z}}^{(L)}}{\partial {w}_{ i}^{(l)}}\right>\Big|
\leq\frac{1}{m^p}\left\|\Theta^{(L+1)}\right\|_2\left\|\frac{\partial {\mathbf{z}}^{(L)}}{\partial {w}_{ i}^{(l)}}\right\|_2
\\=&
\frac{1}{m^p}\left\|\Theta^{(L+1)}\right\|_2\left\|\frac{\partial {\mathbf{z}}^{(L)}}{\partial \mathbf{z}^{(L-1)}}\frac{\partial {\mathbf{z}}^{(L-1)}}{\partial {w}_{ i}^{(l)}}\right\|_2
\leq
\frac{1}{m^p}\left\|\Theta^{(L+1)}\right\|_2\left\|\frac{\partial {\mathbf{z}}^{(L)}}{\partial\mathbf{z}^{(L-1)}}\right\|_F\left\|\frac{\partial {\mathbf{z}}^{(L-1)}}{\partial {w}_{ i}^{(l)}}\right\|_2
\\\leq&
\frac{1}{m^p}\left\|\Theta^{(L+1)}\right\|_2
\left\|\Theta^{(L)}\right\|_2\left\|\frac{\partial {\mathbf{z}}^{(L-1)}}{\partial {w}_{ i}^{(l)}}\right\|_2\leq\cdots
\leq
\frac{1}{m^p}\left\|\Theta^{(L+1)}\right\|_2\cdots
\left\|\Theta^{(L_C+1)}\right\|_2\left\|\frac{\partial {\mathbf{z}}^{(L_C)}}{\partial {w}_{ i}^{(l)}}\right\|_2
\\\leq&
\frac{1}{m^p}\left\|\Theta^{(L+1)}\right\|_2\cdots
\left\|\Theta^{(L_C+1)}\right\|_2\left\|\frac{\partial {\mathbf{z}}^{(L_C)}}{\partial \mathbf{z}^{(L_C-1)}}\right\|_2\left\|\frac{\partial {\mathbf{z}}^{(L_C-1)}}{\partial {w}_{ i}^{(l)}}\right\|_2
\\\leq&\frac{1}{m^p}\left\|\Theta^{(L+1)}\right\|_2\cdots
\left\|\Theta^{(L_C+1)}\right\|_2\left\|\mathbf{P}^{(L_C)}\right\|_2\left\|\mathbf{\Tilde{W}}^{(L_C)}\right\|_2\left\|\frac{\partial {\mathbf{z}}^{(L_C-1)}}{\partial {w}_{ i}^{(l)}}\right\|_2
\\\leq&
\frac{1}{m^p}\left\|\Theta^{(L+1)}\right\|_2\cdots
\left\|\Theta^{(L_C+1)}\right\|_2\frac{1}{\sqrt{s_{L_C}}}\sqrt{s_{L_C}}\left\|\mathbf{w}^{(L_C)}\right\|_2\left\|\frac{\partial {\mathbf{z}}^{(L_C-1)}}{\partial {w}_{ i}^{(l)}}\right\|_2
\\=&\frac{1}{m^p}\left\|\Theta^{(L+1)}\right\|_2\cdots
\left\|\Theta^{(L_C+1)}\right\|_2\left\|\Theta^{(L_C)}\right\|_2\left\|\frac{\partial {\mathbf{z}}^{(L_C-1)}}{\partial {w}_{ i}^{(l)}}\right\|_2\leq\cdots
\\\leq&
\frac{1}{m^p}\left\|\Theta^{(L+1)}\right\|_2\cdots
\left\|\Theta^{(l+1)}\right\|_2\left\|\frac{\partial {\mathbf{z}}^{(l)}}{\partial {w}_{ i}^{(l)}}\right\|_2
.
\end{aligned}
\]
\[
\begin{gathered}
 {\rm From\ }
 z_k^{(l)}=\frac{1}{s_l}\Big(\sigma({\mathbf{w}^{(l)}}^{\top}\mathbf{z}_{[(k-1)s_l+1,\cdots,ks_l]}^{(l-1)})+\cdots+
 \sigma({\mathbf{w}^{(l)}}^{\top}\mathbf{z}_{[ks_l,\cdots,(k+1)s_l-1]}^{(l-1)})\Big)
 ,{\ \rm we\ have:}\\
 \left\|\frac{\partial {\mathbf{z}}^{(l)}}{\partial {w}_{ i}^{(l)}}\right\|_2^2=\sum_{k}\Big(\frac{\partial z_k^{(l)}}{\partial {w}_{ i}^{(l)}}\Big)^2\leq
 \sum_{k}\frac{1}{s_l^2}s_l\Big({z_{(k-1)s_l+i}^{(l-1)}}^2+\cdots+{z_{ks_l+i-1}^{(l-1)}}^2\Big)^2=\frac{1}{s_l}\left\|\mathbf{z}^{(l-1)}\right\|_2^2.
 \end{gathered}
\]
\[
\begin{aligned}
 {\rm Thus,\ }
 &\left\|\frac{\partial f(\mathbf{x};\Theta)}{\partial \Theta^{(l)}}\right\|_2
 =\left\|\frac{\partial f(\mathbf{x};\Theta)}{\partial \mathbf{w}^{(l)}}\right\|_2
 =\Bigg(\sum_{i}\Big(\frac{\partial f(\mathbf{x};\Theta)}{\partial {w}_{ i}^{(l)}}\Big)^2\Bigg)^{\frac{1}{2}}
 \\\leq&
 \frac{1}{m^p}\left\|\Theta^{(L+1)}\right\|_2\cdots
\left\|\Theta^{(l+1)}\right\|_2\Bigg(\sum_{i}\left\|\frac{\partial {\mathbf{z}}^{(l)}}{\partial {w}_{ i}^{(l)}}\right\|_2^2\Bigg)^{\frac{1}{2}}
\\\leq&
 \frac{1}{m^p}\left\|\Theta^{(L+1)}\right\|_2\cdots
\left\|\Theta^{(l+1)}\right\|_2\Bigg(\sum_{i}s_l\frac{1}{s_l}\left\|\mathbf{z}^{(l-1)}\right\|_2^2\Bigg)^{\frac{1}{2}}
\\=&
\frac{1}{m^p}\left\|\Theta^{(L+1)}\right\|_2\cdots
\left\|\Theta^{(l+1)}\right\|_2\left\|\mathbf{z}^{(l-1)}\right\|_2
\\\leq&
\frac{1}{m^p}\left\|\Theta^{(L+1)}\right\|_2\cdots
\left\|\Theta^{(l+1)}\right\|_2\left\|\Theta^{(l-1)}\right\|_2\cdots\left\|\Theta^{(1)}\right\|_2\left\|\mathbf{x}\right\|_2
\leq\frac{1}{m^p L^{\frac{L}{2}}}\left\|\mathbf{x}\right\|_2\left\|\Theta\right\|_2^L.
\end{aligned}\]

\end{proof}

\newpage

\section{\textbf{Proof Details of Section \ref{section: Rademacher estimation} }}\label{section: proof Rademacher complexity}
We derive our basic Rademacher compelexity estimation for deep CNN model with the help of Lemma 1 and Lemma 2 in \citep{golowich2018size} which   performs the peeling technique inside the exp function. For completeness, we write the two lemmas in the follwing lemma.
\begin{lemma}[\rm Lemma 1, Lemma 2 in \citep{golowich2018size}]\label{norm appendix}\ \\
Let $\sigma(\cdot)$ be a $1$-Lipschitz, positive-homogeneous activation function which is applied element-wise. Let $\epsilon=(\epsilon_1,\cdots,\epsilon_n)\sim\mathbb{U}(\{\pm 1\}^n)$. Then for any class of vector-valued function $\mathcal{F}$, and any convex and monotonically increasing function $g:\mathbb{R}\to[0,\infty)$, we have:
\begin{equation}\label{2-norm lemma appendix}
    \mathbb{E}_\epsilon\sup\limits_{\mathbf{f}\in\mathcal{F},\left\|\mathbf{W}\right\|_F\leq R}g\Big(\left\|\sum\limits_{i=1}^n\epsilon_i\sigma(\mathbf{W}\mathbf{f}(\mathbf{x}_i))\right\|_2\Big)\leq2\mathbb{E}_\epsilon\sup\limits_{\mathbf{f}\in\mathcal{F}}g\Big(R\left\|\sum\limits_{i=1}^n\epsilon_i\mathbf{f}(\mathbf{x}_i)\right\|_2\Big),
\end{equation}
\begin{equation}\label{infty-norm lemma appendix}
    \mathbb{E}_\epsilon\sup\limits_{\mathbf{f}\in\mathcal{F},\left\|\mathbf{W}\right\|_{1,\infty}\leq R}g\Big(\left\|\sum\limits_{i=1}^n\epsilon_i\sigma(\mathbf{W}\mathbf{f}(\mathbf{x}_i))\right\|_\infty\Big)\leq2\mathbb{E}_\epsilon\sup\limits_{\mathbf{f}\in\mathcal{F}}g\Big(R\left\|\sum\limits_{i=1}^n\epsilon_i\mathbf{f}(\mathbf{x}_i)\right\|_\infty\Big).
\end{equation} 
\end{lemma}

\begin{proof}[\bfseries Proof of Lemma \ref{Estimation of Rademacher Complexity}]\ \\
\textbf{Deep FNNs.}
Theorem 1 in \citep{golowich2018size} has analyzed deep FNN model using Lemma \ref{norm appendix}:
\[
{\rm Rad}_n(\mathcal{F}_{\mathbf{Q}})
\leq
   \frac{\sqrt{2(L+1)\log2}+1}{m^p\sqrt{n}}\Big(\prod\limits_{l=1}^{L+1}Q_l\Big).
\]

\textbf{Deep CNNs.}
As an extension of it, we use Lemma \ref{norm appendix} to analyze deep CNN model especially for the convolution layers. 

Let $\mathcal{F}_{\mathbf{Q}}=\Big\{f(\cdot,\Theta):\left\|\Theta^{(l)}\right\|_2\leq Q_l,\forall l\in[L+1]\Big\}$ where $\mathbf{Q}=(Q_1,\cdots,Q_{L+1})^\top$.

Fix $\lambda>0$ to be chosen later, and we use  the similar technique in the proof of Theorem 1 and Theorem 2 in \citep{golowich2018size}. 

$\bullet$ \textbf{Fully-connected layers.}  By using (\ref{2-norm lemma appendix}) in Lemma \ref{norm appendix}, for fully-connected layer $l\in[L]-[L_C+1]$ we have the estimation for adjacent layers:
\[
\begin{aligned}
\mathbb{E}_\epsilon\sup\limits_{\mathbf{z}^{(l)}}\exp\Big(\lambda\left\|\sum\limits_{i=1}^n\epsilon_i\mathbf{z}^{(l)}(\mathbf{x}_i)\right\|_2\Big)&\leq
    \mathbb{E}_\epsilon\sup\limits_{\left\|\Theta^{(l)}\right\|_2\leq Q_l,\mathbf{z}^{(l-1)}}\exp\Big(\lambda\left\|\sum\limits_{i=1}^n\epsilon_i\sigma(\mathbf{A}^{(l)}\mathbf{z}^{(l-1)}(\mathbf{x}_i))\right\|_2\Big)
    \\&=
    \mathbb{E}_\epsilon\sup\limits_{\left\|\mathbf{A}^{(l)}\right\|_F\leq Q_l,\mathbf{z}^{(l-1)}}\exp\Big(\lambda\left\|\sum\limits_{i=1}^n\epsilon_i\sigma(\mathbf{A}^{(l)}\mathbf{z}^{(l-1)}(\mathbf{x}_i))\right\|_2\Big)
    \\&\leq
    2\mathbb{E}_\epsilon\sup\limits_{\mathbf{z}^{(l-1)}}\exp\Big(\lambda Q_l\left\|\sum\limits_{i=1}^n\epsilon_i\mathbf{z}^{(l-1)}(\mathbf{x}_i)\right\|_2\Big).
\end{aligned}
\]
So we have the estimation for fully-connected layers:
\begin{equation}\label{fully connected layers in proof of Rademacher}
\begin{aligned}
    &\mathbb{E}_\epsilon\sup\limits_{f(\cdot,\Theta)\in\mathcal{F_\mathbf{Q}}}\exp\Big(\lambda\left|\sum\limits_{i=1}^n\epsilon_i f(\mathbf{x}_i;\Theta)\right|\Big)\\
    \leq&2\mathbb{E}_\epsilon\sup\limits_{\mathbf{z}^{(L)}}\exp\Big(\lambda\frac{1}{m^p}Q_{L+1}\left\|\sum\limits_{i=1}^n\epsilon_i\mathbf{z}^{(L)}(\mathbf{x}_i)\right\|_2\Big)\leq\cdots
    \\\leq&2^{L+1-L_C}\mathbb{E}_\epsilon\sup\limits_{\mathbf{z}^{(L_C)}}\exp\Big(\lambda\frac{1}{m^p}\Big(\prod\limits_{l=L_C+1}^{L+1}Q_l\Big)\left\|\sum\limits_{i=1}^n\epsilon_i\mathbf{z}^{(L_C)}(\mathbf{x}_i)\right\|_2\Big)
    .
\end{aligned}
\end{equation}

$\bullet$ \textbf{Convolutional layers.}\  We use the same notations in the proof of Property \ref{property: high order upper bound}, then we have the norm inequations:
\begin{equation*}
\begin{gathered}
 \left\|\Tilde{\mathbf{W}}^{(l)}\right\|_{1,\infty}=\left\|\mathbf{w}^{(l)}\right\|_1=\left\|\Theta^{(l)}\right\|_1\leq\sqrt{s_l}\left\|\Theta^{(l)}\right\|_2,
 \\
 \left\|\mathbf{z}^{(l)}\right\|_{\infty}\leq\left\|\mathbf{y}^{(l)}\right\|_{\infty}.
\end{gathered}
\end{equation*}
We define $\lambda'=\lambda\frac{1}{m^p}\big(\prod\limits_{l=L_C+1}^{L+1}Q_l\big)$ for convenience. Then for any convolution layer $l\in[L_C]$, we have the estimation between two adjacent layers by (\ref{infty-norm lemma appendix}) in Lemma \ref{norm appendix}:
\begin{equation*}
\begin{aligned}
    \mathbb{E}_\epsilon\sup\limits_{\mathbf{z}^{(l)}}\exp\Big(\lambda'\left\|\sum\limits_{i=1}^n\epsilon_i\mathbf{z}^{(l)}(\mathbf{x}_i)\right\|_\infty\Big)
    &=        \mathbb{E}_\epsilon\sup\limits_{\mathbf{y}^{(l)}}\exp\Big(\lambda'\left\|\sum\limits_{i=1}^n\epsilon_i\mathbf{P}^{(l)}\mathbf{y}^{(l)}(\mathbf{x}_i)\right\|_\infty\Big)
    \\
    &\leq
        \mathbb{E}_\epsilon\sup\limits_{\mathbf{y}^{(l)}}\exp\Big(\lambda'\left\|\sum\limits_{i=1}^n\epsilon_i\mathbf{y}^{(l)}(\mathbf{x}_i)\right\|_\infty\Big)
        \\&=
  \mathbb{E}_\epsilon\sup\limits_{\left\|\Theta^{(l)}\right\|_2\leq Q_l,\mathbf{z}^{(l-1)}}\exp\Big(\lambda'\left\|\sum\limits_{i=1}^n\epsilon_i\sigma(\Tilde{\mathbf{W}}^{(l)}\mathbf{z}^{(l-1)}(\mathbf{x}_i))\right\|_\infty\Big)
  \\&\leq
    \mathbb{E}_\epsilon\sup\limits_{\left\|\Tilde{\mathbf{W}}^{(l)}\right\|_{1,\infty}\leq \sqrt{s_l}Q_l,\mathbf{z}^{(l-1)}}\exp\Big(\lambda'\left\|\sum\limits_{i=1}^n\epsilon_i\sigma(\Tilde{\mathbf{W}}^{(l)}\mathbf{z}^{(l-1)}(\mathbf{x}_i))\right\|_\infty\Big)
      \\&\leq
    2\mathbb{E}_\epsilon\sup\limits_{\mathbf{z}^{(l-1)}}\exp\Big(\lambda'\sqrt{s_l}Q_l\left\|\sum\limits_{i=1}^n\epsilon_i\mathbf{z}^{(l-1)}(\mathbf{x}_i)\right\|_\infty\Big).
\end{aligned}
\end{equation*}
So we have the estimation for convolution layers:
\begin{equation}\label{convolution layers in proof of Rademacher}
    \begin{aligned}
    &\mathbb{E}_\epsilon\sup\limits_{\mathbf{z}^{(L_C)}}\exp\Big(\lambda'\left\|\sum\limits_{i=1}^n\epsilon_i\mathbf{z}^{(L_C)}(\mathbf{x}_i)\right\|_\infty\Big)\\\leq&
    2\mathbb{E}_\epsilon\sup\limits_{\mathbf{z}^{(L_C-1)}}\exp\Big(\lambda'\left\|\sum\limits_{i=1}^n\epsilon_i\mathbf{z}^{(L_C-1)}(\mathbf{x}_i)\right\|_\infty\Big)\leq\cdots\\
    \leq&
    2^{L_C}\mathbb{E}_\epsilon\exp\Big(\lambda'\Big(\prod\limits_{l=1}^{L_C}\sqrt{s_l}Q_l\Big)\left\|\sum\limits_{i=1}^n\epsilon_i\mathbf{x}_i\right\|_\infty\Big)
    \end{aligned}.
\end{equation}
Combining (\ref{fully connected layers in proof of Rademacher}) and (\ref{convolution layers in proof of Rademacher}), we obtain the estimation for all layers:
\begin{equation}\label{all layers in proof of Rademacher}
    \begin{aligned}
        &\mathbb{E}_\epsilon\sup\limits_{f(\cdot,\Theta)\in\mathcal{F_\mathbf{Q}}}\exp\Big(\lambda\left|\sum\limits_{i=1}^n\epsilon_i f(\mathbf{x}_i;\Theta)\right|\Big)
        \\\leq&
        2^{L+1-L_C}\mathbb{E}_\epsilon\sup\limits_{\mathbf{z}^{(L_C)}}\exp\Big(\lambda\frac{1}{m^p}\Big(\prod\limits_{l=L_C+1}^{L+1}Q_l\Big)\left\|\sum\limits_{i=1}^n\epsilon_i\mathbf{z}^{(L_C)}(\mathbf{x}_i)\right\|_2\Big)
        \\\leq& 2^{L+1-L_C}\mathbb{E}_\epsilon\sup\limits_{\mathbf{z}^{(L_C)}}\exp\Big(\lambda\frac{1}{m^p}\Big(\prod\limits_{l=L_C+1}^{L+1}Q_l\Big)\sqrt{m_{L_C}}\left\|\sum\limits_{i=1}^n\epsilon_i\mathbf{z}^{(L_C)}(\mathbf{x}_i)\right\|_\infty\Big)
        \\\leq&
        2^{L+1}\mathbb{E}_\epsilon\exp\Big(\lambda\Big(\prod\limits_{l=1}^{L_C}\sqrt{s_l}\Big)\Big(\prod\limits_{l=1}^{L+1}Q_L\Big)\frac{\sqrt{m_{L_C}}}{m^p}\left\|\sum\limits_{i=1}^n\epsilon_i\mathbf{x}_i\right\|_\infty\Big).
    \end{aligned}
\end{equation}
With (\ref{all layers in proof of Rademacher}), we have the estimation of the Rademacher complexity:
\[
\begin{aligned}
    n{\rm Rad}_n(\mathcal{F}_{\mathbf{Q}})
    &=\mathbb{E}_\epsilon\sup\limits_{f(\cdot,\Theta)\in\mathcal{F_\mathbf{Q}}}\sum\limits_{i=1}^n\epsilon_i f(\mathbf{x}_i;\Theta)
    \\&\leq\frac{1}{\lambda}\log\mathbb{E}_\epsilon\sup\limits_{f(\cdot,\Theta)\in\mathcal{F_\mathbf{Q}}}\exp\Big(\lambda\sum\limits_{i=1}^n\epsilon_i f(\mathbf{x}_i;\Theta)\Big)
    \\&\leq\frac{1}{\lambda}\log\mathbb{E}_\epsilon\sup\limits_{f(\cdot,\Theta)\in\mathcal{F_\mathbf{Q}}}\exp\Big(\lambda\left|\sum\limits_{i=1}^n\epsilon_i f(\mathbf{x}_i;\Theta)\right|\Big)\\&\leq\frac{1}{\lambda}\log        2^{L+1}\mathbb{E}_\epsilon\exp\Big(\lambda\Big(\prod\limits_{l=1}^{L_C}\sqrt{s_l}\Big)\Big(\prod\limits_{l=1}^{L+1}Q_L\Big)\frac{\sqrt{m_{L_C}}}{m^p}\left\|\sum\limits_{i=1}^n\epsilon_i\mathbf{x}_i\right\|_\infty\Big).
\end{aligned}\]
Now we define $Z=\Big(\prod\limits_{l=1}^{L_C}\sqrt{s_l}\Big)\Big(\prod\limits_{l=1}^{L+1}Q_L\Big)\frac{\sqrt{m_{L_C}}}{m^p}\left\|\sum\limits_{i=1}^n\epsilon_i\mathbf{x}_i\right\|_\infty\Big)$ 
and $\mathbf{X}=(\mathbf{x}_1,\cdots,\mathbf{x}_n)^\top=(x_{i j})_{n\times d}$. If we choose
\[
\lambda =\frac{\sqrt{L+2+\log d}}{\Big(\prod\limits_{l=1}^{L_C}\sqrt{s_l}\Big)\Big(\prod\limits_{l=1}^{L+1}Q_L\Big)\frac{\sqrt{m_{L_C}}}{m^p}\sqrt{\max\limits_{j\in[d]}\Big(\sum\limits_{i=1}^n\left|x_{i j}\right|^2}\Big)},
\]
we have the Rademacher complexity result by the proof of Theorem 2 in \citep{golowich2018size}:
	\[
	   {\rm Rad}_n(\mathcal{F}_{\mathbf{Q}})
	   \leq\frac{2}{n m^p}{\sqrt{\Big(L+2+\log d\Big)\max\limits_{j\in[d]}\Big(\sum\limits_{i=1}^n\left|x_{i j}\right|^2}\Big)}\Big(\sqrt{m_{L_C}}\prod\limits_{l=1}^{L_C}\sqrt{s_{l}}\Big)\Big(\prod\limits_{l=1}^{L+1}Q_l\Big).
	\]
Considering the convolutional scale condition
\[
m_{L_C}\Big(\prod_{l=1}^{L_C}s_l\Big)\leq m_0=d,
\]
we have:
\[
\begin{aligned}
    d\max\limits_{j\in[d]}\sum\limits_{i=1}^n x_{i j}^2
    &\leq d\sum_{i=1}^n \max_{j\in[d]}x_{i j}^2=d\sum_{i=1}^n\max_{j\in[d]}\mathbf{e}_j^\top\mathbf{xx^\top}\mathbf{e}_j\\
    &\leq d\sum_{i=1}^n \lambda_{\max}(\mathbf{xx}^\top)=d n{\rm tr}(\mathbf{xx}^\top)\\
    &=d n\left\|\mathbf{x}\right\|_2\leq d n.
\end{aligned}
\]
Now we can simplify the Rademacher bound:
\[
\begin{aligned}
    {\rm Rad}_n(\mathcal{F}_{\mathbf{Q}})
    &\leq\frac{2\Big(\prod\limits_{l=1}^{L+1}Q_l\Big)}{n m^p}{\sqrt{\Big(L+2+\log d\Big)d\max\limits_{j\in[d]}\Big(\sum\limits_{i=1}^n\left|x_{i j}\right|^2}\Big)}
    \\&\leq
    \frac{2\Big(\prod\limits_{l=1}^{L+1}Q_l\Big)}{n m^p}{\sqrt{\Big(L+2+\log d\Big)n d}}
    \\&=
   \frac{2\sqrt{(L+2+\log d)d}}{m^p\sqrt{n}}\Big(\prod\limits_{l=1}^{L+1}Q_l\Big).
\end{aligned}
\]

Given all of that, we obtain the Radmacher complexity estimation:
\[
	{\rm Rad}_n(\mathcal{F}_{\mathbf{Q}})\leq
	\frac{C_{L,d}}{m^p\sqrt{n}}\Big(\prod\limits_{l=1}^{L+1}Q_l\Big),
\]
where $C_{L,d}=\sqrt{2(L+1)\log 2}+1$ for deep FNNs and $C_{L,d}=2\sqrt{(L+2+\log d)d}$ for deep CNNs.

\end{proof}

\newpage

\section{Proof Details of GF}\label{section: proof GF}
The first lemma is about the random initializtaion (\ref{random initialization}) for small initial norms with high probability.
\begin{lemma}\label{lemma: random initialization}
Let $\Theta(0)$ be obtained by random initialization (\ref{random initialization}), then with probability at least $1-\delta$, we have:
\[
\left\|\Theta^{(l)}(0)\right\|_2^2\leq\kappa^2+\kappa^2\max\Big\{\frac{4}{q(l)}\log(\frac{1}{\delta}),\sqrt{\frac{8}{q(l)}\log(\frac{1}{\delta})\big)}\Big\},\  \forall l\in[L+1].
\]
\end{lemma}

\begin{proof}[\bfseries Proof of Lemma
\ref{lemma: random initialization}]
\ \\
Recalling the notation in model (table \ref{deep CNN}), we have $\Theta^{(l)}={\rm vec}(\mathbf{A}^{(l)})$ for $l\in[L]-[L_C]$ and $\Theta^{(l)}=\mathbf{w}^{(l)}$ for $l\in [L_C]$.
	
	For fixed $l\in[L]$, we define random variable $X_1,\cdots\,X_{q(l)}$ as $(X_1,\cdots,X_{q(l)})^\top=\sqrt{\frac{q(l)}{\kappa^2}}\Theta^{(l)}(0)$.

	It is easy to verify that $\mathbb{E}X_{i}=1$ and $X_i-\mathbb{E}X_{i}$ is $S(2,4)$ sub exponential, i.e.
	\[
	\mathbb{E}\exp(\lambda (X_i-1))\leq\exp(\frac{2^2\lambda^2}{2}),\ \forall |\lambda|<\frac{1}{4}.\]
	By Bernstein Inequation and $\sum\limits_{i=1}^{q(l)} X_i^2=\frac{q(l)}{\kappa^2}\left\|\Theta^{(l)}(0)\right\|_2^2$, we have:
	\[
	\begin{aligned}
	 \mathbb{P}\Big(\left\|\Theta^{(l)}(0)\right\|_2^2-\kappa^2>\kappa^2 t\Big)
	 &=\mathbb{P}\Big(\frac{\kappa^2}{q(l)}\sum\limits_{i=1}^{q(l)} X_{i}^2-\kappa^2>\kappa^2 t\Big)\\
	 &=\mathbb{P}\Big(\frac{1}{q(l)}\sum\limits_{i=1}^{q(l)} X_{i}^2-1> t\Big)\\
	 &\leq
	 \exp\Big(-\frac{q(l)}{2}\min\Big\{\frac{t}{2},\frac{t^2}{4}\Big\}\Big).
	\end{aligned}
	\]
	So we have
	\[
	\begin{gathered}
	 \mathbb{P}\Big(\left\|\Theta^{(l)}(0)\right\|_2^2>\kappa^2+\kappa^2\frac{4}{q(l)}\log(\frac{1}{\delta})\Big)
	 \leq \delta,\ \forall\delta<\exp\Big(-\frac{q(l)}{2}\Big);\\
	 \mathbb{P}\Big(\left\|\Theta^{(l)}(0)\right\|_2^2>\kappa^2+\kappa^2\sqrt{\frac{8}{q(l)}\log(\frac{1}{\delta})}\Big)
	 \leq \delta,\ \forall\delta\geq\exp\Big(-\frac{m_{l-1}m_l}{2}\Big).
	\end{gathered}
	\]
So we obtain the bound with high probability:
	\[
	\mathbb{P}\Big(\left\|\Theta^{(l)}(0)\right\|_2^2>\kappa^2+\kappa^2\max\Big\{\frac{4}{q(l)}\log(\frac{1}{\delta}),\sqrt{\frac{8}{q(l)}\log(\frac{1}{\delta})\big)}\Big\}\Big)
	\leq \delta,\ \forall\delta\in(0,1),\ \forall l\in[L+1].
	\]
So with probability at least $1-L\delta$, we have:
\[
\left\|\Theta^{(l)}(0)\right\|_2^2\leq\kappa^2+\kappa^2\max\Big\{\frac{4}{q(l)}\log(\frac{1}{\delta}),\sqrt{\frac{8}{q(l)}\log(\frac{1}{\delta})\big)}\Big\},\  \forall l\in[L+1].
\]
Substituting $\delta$ with $\delta/(L+1)$, with probability at least $1-\delta$ we have:
\[
\left\|\Theta^{(l)}(0)\right\|_2^2\leq \kappa^2+\kappa^2\max\Big\{\frac{4}{q(l)}\log(\frac{L}{\delta}),\sqrt{\frac{8}{q(l)}\log(\frac{L}{\delta})\big)}\Big\},\  \forall l\in[L+1].
\]
\end{proof}
\begin{definition} We analyze generalization error through dynamic hypothesis space, which is defined as:
\[
\mathcal{F}(T)=\bigcup\limits_{t=1}^T\Big\{f(\cdot,\Theta(t)): \{\Theta(s)\}_{s=0}^t {\rm\ are\  trained\  by\  optimization\  algorithm}\Big\},
\]
and we also use $\mathcal{F}=\bigcup\limits_{T}\mathcal{F}{(T)}$ to denote the total hypothesis space.
\end{definition}

The crucial step in the proof of Theorem \ref{continuous GD thm} is to estimate $l_2$ norm dynamics of parameters in each layer, now we give the detailed proof.

\begin{proof}[\bfseries Proof of theorem \ref{continuous GD thm}]\ \\
First, we will estimate the norm dynamics of each layer: 
\[
	\left\|\Theta^{(l)}{(t)}\right\|_2^2\leq\left\|\Theta^{(l)}{(0)}\right\|_2^2+\int_{0}^t2\sqrt{2\mathcal{L}_n(\Theta{(t)})}\Big(C_y-\sqrt{2\mathcal{L}_n(\Theta{(t)})}\Big)\mathrm{d}t,\ \forall l\in[L+1].
\]

For any $l\in[L+1]$, it's easy to verify the following dynamics by Property \ref{multiplicative property}:
\begin{equation}\label{equ: extension discussion}
\begin{aligned}
	\frac{\mathrm{d} \left\|\Theta^{(l)}(t)\right\|^2}{\mathrm{d} t}=&
	-2\left<\Theta^{(l)}(t),\frac{\partial \mathcal{L}_n(\Theta(t))}{\partial \Theta^{(l)}}\right>
	\\=&	-\frac{2}{n}\sum_{i=1}^n\Big(f(\mathbf{x}_i;\Theta(t))-y_i\Big)\left<\Theta^{(l)}(t),\frac{\partial f(\mathbf{x}_i;\Theta(t))}{\partial \Theta^{(l)}}\right>
	\\=&-\frac{2}{n}\sum_{i=1}^n\Big(f(\mathbf{x}_i;\Theta{(t)})-y_i\Big)f(\mathbf{x}_i;\Theta{(t)})
	\\=&-4\mathcal{L}_n(\Theta{(t)})-\frac{2}{n}\sum_{i=1}^n\Big(f(\mathbf{x}_i;\Theta{(t)})-y_i\Big)y_i
	\\\leq&
	-4\mathcal{L}_n(\Theta{(t)})+\frac{2}{n}\sqrt{\sum_{i=1}^n\Big(f(\mathbf{x}_i;\Theta{(t)})-y_i\Big)^2}\sqrt{\sum_{i=1}^n y_i^2}
	\\\leq&-4\mathcal{L}_n(\Theta{(t)})+2C_y\sqrt{2\mathcal{L}_n(\Theta{(t)})}
	\\=&2\sqrt{2\mathcal{L}_n(\Theta{(t)})}\Big(C_y-\sqrt{2\mathcal{L}_n(\Theta{(t)})}\Big),\ \forall l\in[L+1].
\end{aligned}
\end{equation}
Integrating the above formula, we obtain:
\[
\left\|\Theta^{(l)}(t)\right\|^2\leq\left\|\Theta^{(l)}(0)\right\|^2+\int_{0}^t2\sqrt{2\mathcal{L}_n(\Theta{(s)})}\Big(C_y-\sqrt{2\mathcal{L}_n(\Theta{(s)})}\Big)\mathrm{d}s,\ \forall l\in[L+1].
\]
Then with Lemma \ref{Estimation of Rademacher Complexity}, we have the Rademacher complexity estimation:
\[
\begin{aligned}
    &{\rm Rad}_n(\mathcal{F}(T))
    \\\leq&
    \frac{C_{L,d}}{m^p\sqrt{n}}\prod\limits_{l=1}^{L+1}\left\|\Theta^{(l)}(T)\right\|_2
    \\\leq&
     \frac{C_{L,d}}{m^p\sqrt{n}}\prod\limits_{l=1}^{L+1}\Bigg(\left\|\Theta^{(l)}(0)\right\|_2^2 +\int_{0}^T2\sqrt{2\mathcal{L}_n(\Theta{(t)})}\Big(C_y-\sqrt{2\mathcal{L}_n(\Theta{(t)})}\Big)\mathrm{d}t\Bigg)^{\frac{1}{2}}.
\end{aligned}
\]

Combining the formulation above with Lemma \ref{lemma: random initialization}, with probability at least $1-\delta$ we have:
\[
\begin{aligned}
    {\rm Rad}_n(\mathcal{F}(T))
    \leq
     \frac{C_{L,d}}{m^p\sqrt{n}}\Bigg(&\kappa^2+\kappa^2\max\Big\{\frac{4}{q(l)}\log(\frac{L}{\delta}),\sqrt{\frac{8}{q(l)}\log(\frac{L}{\delta})\big)}\Big\} \\&+\int_{0}^T2\sqrt{2\mathcal{L}_n(\Theta{(t)})}\Big(C_y-\sqrt{2\mathcal{L}_n(\Theta{(t)})}\Big)\mathrm{d}t\Bigg)^{\frac{L+1}{2}}.
\end{aligned}
\]

So from Lemma \ref{lemma: gen and Rad}, with probability at least $1-2\delta$, we obtain:
\[
\begin{aligned}
\mathcal{E}_{\rm gen}(T)
    \lesssim\frac{C_{L,d}}{m^p\sqrt{n}}\Bigg(&\kappa^2+\kappa^2\max\Big\{\frac{4}{q(l)}\log(\frac{L}{\delta}),\sqrt{\frac{8}{q(l)}\log(\frac{L}{\delta})\big)}\Big\}
    \\&+\int_{0}^T2\sqrt{2\mathcal{L}_n(\Theta{(t)})}\Big(C_y-\sqrt{2\mathcal{L}_n(\Theta{(t)})}\Big)\mathrm{d}t\Bigg)^{\frac{L+1}{2}}
    +\sqrt{\frac{\log(1/\delta)}{n}}
        \\=
    \frac{C_{L,d}}{m^p\sqrt{n}}\Bigg(&\mathcal{O}\Big(\kappa^2\Big)+\int_{t=0}^T2\eta_t\sqrt{2\mathcal{L}_n(\Theta{(t)})}\Big(C_y-\sqrt{2\mathcal{L}_n(\Theta{(t)})}\Big)\mathrm{d}t\Bigg)^{\frac{L+1}{2}}
    \\+&\sqrt{\frac{\log(1/\delta)}{n}}.
\end{aligned}
\]
Substituting $\delta$ with $\delta/2$, we obtain this theorem.
\end{proof}

\newpage

\section{Proof Details of GD}\label{section: proof GD}
\begin{lemma}\label{lemma: log}
Let $x\geq 1$ and $L\in\mathbb{N}_{+}$, then for any $\epsilon\in(0,1)$ we have:
\[
1\leq\log x+1\leq \frac{L}{1-\epsilon}x^{\frac{1-\epsilon}{L
}}.
\]
\end{lemma}
\begin{proof}[Proof of Lemma \ref{lemma: log}]\ \\
For any $x\geq 1$ and $L\geq1$, we have:
\[
\begin{aligned}
    1\leq\log x+1={\frac{L}{1-\epsilon}}\log x^{\frac{1-\epsilon}{L}}+1
    \leq
    {\frac{L}{1-\epsilon}}x^ {\frac{1-\epsilon}{L}}-\frac{L}{1-\epsilon}+1\leq 
    {\frac{L}{1-\epsilon}}x^ {\frac{1-\epsilon}{L}}.
\end{aligned}\]
\end{proof}

\begin{lemma}\label{lemma: T sum}
For any $T\in\mathbb{N}$, $n\in\mathbb{N}_{+}$ and $\alpha\in(0,1)$, we have:
\[
\begin{gathered}
    \sum_{k=0}^{T}\frac{1-\alpha}{(t+1)^\alpha}\leq(T+1)^{1-\alpha},\\
    \sum_{k=0}^{T}\frac{1}{k+1}\leq\log(T+1) +1,\\
    \frac{\alpha}{(n+1)^{\alpha+1}}\leq\frac{1}{n^{\alpha}}-\frac{1}{(n+1)^{\alpha}}.
\end{gathered}
\]
\end{lemma}

 The crucial step in the proof of Theorem \ref{thm disc GD 2NN}  and Theorem \ref{thm disc SGD 2NN} is also to estimate $l_2$ norm dynamics of parameters in each layer, but we need fine-grained analysis in this more complex case.

\begin{proof}[\bfseries 
Proof of Theorem \ref{thm disc GD 2NN}]\ \\
First, we estimate the norm dynamics of each layer: 
	\[
	\begin{aligned}		&\left\|\Theta^{(l)}{(t+1)}\right\|_2^2
		\\=&\left\|\Theta^{(l)}{(t)}\right\|_2^2+\left\|\Theta^{(l)}{(t+1)}-\Theta^{(l)}{(t)}\right\|_2^2+2\left<\Theta^{(l)}{(t)},\Theta^{(l)}{(t+1)}-\Theta^{(l)}{(t)}\right>
		\\=&\left\|\Theta^{(l)}{(t)}\right\|_2^2+\eta_t^2\left\|\frac{\partial\mathcal{L}(\Theta{(t)})}{\partial \Theta^{(l)}}\right\|_2^2-2\eta_t\left<\Theta^{(l)}{(t)},\frac{\partial\mathcal{L}(\Theta{(t)})}{\partial \Theta^{(l)}}\right>
		\\\leq&\left\|\Theta^{(l)}{(t)}\right\|_2^2+\eta_t^2(C_f+C_y)^2\sup\limits_{\mathbf{x}}\left\|\frac{ \partial f(\mathbf{x};\Theta{(t)})}{\partial \Theta^{(l)}}\right\|_2^2-\eta_t\frac{2}{n}\sum_{i=1}^n(f(\mathbf{x}_i;\Theta{(t)})-y_i)\left<\Theta^{(l)}{(t)},\frac{ \partial f(\mathbf{x};\Theta{(t)})}{\partial \Theta^{(l)}}\right>\\
		=&\left\|\Theta^{(l)}{(t)}\right\|_2^2+\eta_t^2(C_f+C_y
		)^2\sup\limits_{\mathbf{x}}\left\|\frac{ \partial f(\mathbf{x};\Theta{(t)})}{\partial \Theta^{(l)}}\right\|_2^2-\eta_t\frac{2}{n}\sum_{i=1}^n(f(\mathbf{x}_i;\Theta{(t)})-y_i)f(\mathbf{x}_i;\Theta{(t)})
		\\
		\leq&
		\left\|\Theta^{(l)}{(t)}\right\|_2^2+\frac{\eta_t^2(C_f+C_y)^2}{m^{2p}L^{L-1}}\left\|\Theta^{(l)}{(t)}\right\|_2^{2L}-\eta_t\frac{2}{n}\sum_{i=1}^n(f(\mathbf{x}_i;\Theta{(t)})-y_i)f(\mathbf{x}_i;\Theta{(t)})
		.
	\end{aligned}
	\]
	Recalling the homogeneity property \ref{multiplicative property}, we have:
	\[
	\begin{aligned}
		&-\frac{2}{n}\sum_{i=1}^n(f(\mathbf{x}_i;\Theta{(t)})-y_i)f(\mathbf{x}_i;\Theta{(t)})\\
		=&-4\mathcal{L}(\Theta{(t)})-\frac{2}{n}\sum_{i=1}^n(f(\mathbf{x}_i;\Theta{(t)})-y_i)y_i\\
		\leq&-4\mathcal{L}(\Theta{(t)}+\frac{2}{n}\sqrt{\sum_{i=1}^n(f(\mathbf{x}_i;\Theta{(t)})-y_i)^2}\sqrt{\sum_{i=1}^n y_i^2}\\
		\leq&2\sqrt{2\mathcal{L}_n(\Theta{(t)})}\Big(C_y-\sqrt{2\mathcal{L}_n(\Theta{(t)})}\Big).
	\end{aligned}
	\]
	Combining the two formulation above, we obtain:
	\begin{equation}\label{fine grained norm inequation}
	\left\|\Theta^{(l)}{(t+1)}\right\|_2^2\leq
	\left\|\Theta^{(l)}{(t)}\right\|_2^2+\frac{\eta_t^2(C_f+C_y)^2}{m^{2p}L^{L-1}}\left\|\Theta^{(l)}{(t)}\right\|_2^{2L}+2\eta_t\sqrt{2\mathcal{L}_n(\Theta{(t)})}\Big(C_y-\sqrt{2\mathcal{L}_n(\Theta{(t)})}\Big).
	\end{equation}
Then we will prove this theorem for two cases respectively.

\textbf{I. The case $\alpha\in(\frac{L+1}{L+2},1).$}

For any $\lambda\in(0,1/\sqrt{3})$, we choose $\eta$:
\begin{equation}\label{eta proof case < 1}
\begin{aligned}
    \eta\leq\min_{l\in[L+1]}\Bigg\{&\frac{\lambda m^{p}L^{\frac{L-1}{2}}}{(C_f+C_y)\left\|\Theta^{(l)}(0)\right\|_2^{L-1}},
    \frac{2(1-\alpha)\lambda^2\left\|\Theta^{(l)}(0)\right\|_2^{2}}{C_y^2 T_0},
    \\&\frac{\lambda m^p L^{\frac{L-1}{2}}\sqrt{(L+2)\alpha-(L+1)}}{(C_f+C_y)T_0^{\frac{(L+2)\alpha-L}{2}}(1+3\lambda^2)^{\frac{L}{2}}\left\|\Theta^{(l)}(0)\right\|_2^{L-1}}
    \Bigg\}.
\end{aligned}
\end{equation}

For the sake of brevity, for any $t\geq 1$, we define:
	\[
	\begin{gathered}
\mathcal{V}_l=\Big(1+2\lambda^2\Big)\left\|\Theta^{(l)}{(0)}\right\|_2^2,\\
\phi(t)=        	    \frac{\lambda^2\left\|\Theta^{(l)}{(0)}\right\|_2^2}{(t+1)^{(L+2)\alpha-(L+1)}},\\ \psi(t)=\sqrt{2\mathcal{L}_n(\Theta{(t)})}\Big(C_y-\sqrt{2\mathcal{L}_n(\Theta{(t)})}\Big).
\end{gathered}
	\]
	It's easy to verify $\psi(t)\leq\frac{1}{4}C_y^2\leq\frac{1}{4},\forall t\geq1.$
	
	Under the above marks, we prove the following estimation by induction:
    \[
	\left\|\Theta^{(l)}{(T)}\right\|_2^2\leq
	\mathcal{V}_l-\phi(T)+\sum_{t=0}^{T-1}2\eta_t\psi(t),\ \forall T\in\mathbb{N}_{+}.
	\]
	For $T=1$, with ($\ref{fine grained norm inequation}$) we have:
	\[
	\begin{aligned}
		\left\|\Theta^{(l)}{(1)}\right\|_2^2&\leq
		\left\|\Theta^{(l)}{(0)}\right\|_2^2+\frac{\eta^2(C_f+C_y)^2}{m^{2p}L^{L-1}}\left\|\Theta^{(l)}{(0)}\right\|_2^{2L}+2\eta\sqrt{2\mathcal{L}_n(\Theta{(0)})}\Big(C_y-\sqrt{2\mathcal{L}_n(\Theta{(0)})}\Big)
		\\&\overset{(\ref{eta proof case < 1})}{\leq}
		(1+\lambda^2)\left\|\Theta^{(l)}{(0)}\right\|_2^2+2\eta\sqrt{2\mathcal{L}_n(\Theta{(0)})}\Big(C_y-\sqrt{2\mathcal{L}_n(\Theta{(0)})}\Big)
		\\&=
		\mathcal{V}_l-\phi(1)+2\eta\psi(0).
	\end{aligned}
	\]
	
		Assume the inequation holds for $1,\cdots,T$, then for $T+1$ we have:
		\[
	\begin{aligned}
		\left\|\Theta^{(l)}{(T+1)}\right\|_2^2\leq&
		\phi(T+1)-\phi(T)+
		\Big(\mathcal{V}_l-\phi(T+1)
		+\sum_{t=0}^{T}2\eta_t\psi(t)\Big)\\&
		+\frac{\eta_T^2(C_f+C_y)^2}{m^{2p}L^{L-1}}\Big(\mathcal{V}_l-\phi(T)
		+\sum_{t=0}^{T-1}2\eta_t\psi(t)\Big)^L.
	\end{aligned}
	\]

where
	\[
	\begin{aligned}
		&\frac{\eta_T^2(C_f+C_y
		)^2}{m^{2p}L^{L-1}}\Big(\mathcal{V}_l-\phi(T)+\sum_{t=0}^{T-1}2\eta_t\psi(t)\Big)^L
		\\\leq&
		\frac{\eta_T^2(C_f+C_y)^2}{m^{2p}L^{L-1}}\Big(\mathcal{V}_l+\frac{C_y^2}{2}\sum_{t=0}^{T-1}\eta_t\Big)^L
		\\=&
		\eta^2\frac{(C_f+C_y)^2}{m^{2p}L^{L-1}\lceil \frac{T+1}{T_0}\rceil^{2\alpha}}\Big(\mathcal{V}_l+\frac{\eta C_y^2}{2}\sum_{t=0}^{T-1}\frac{1}{\lceil\frac{t+1}{T_0}\rceil^\alpha}\Big)^L
		\\\leq&
   	    \eta^2\frac{(C_f+C_y)^2}{m^{2p}L^{L-1}\lceil\frac{T+1}{T_0}\rceil^{2\alpha}}\Big(\mathcal{V}_l+\frac{\eta C_y^2}{2}T_0\sum_{k=1}^{\lceil\frac{T}{T_0}\rceil}\frac{1}{k^\alpha}\Big)^L
   	    \\\overset{\rm Lemma\ \ref{lemma: T sum}}{\leq}&
   	    \eta^2\frac{(C_f+C_y)^2}{m^{2p}L^{L-1}\lceil\frac{T+1}{T_0}\rceil^{2\alpha}}\Big(\mathcal{V}_l+\frac{\eta C_y^2 T_0}{2(1-\alpha)}\lceil\frac{T}{T_0}\rceil^{1-\alpha}\Big)^L
	\\\leq&
		\eta^2\frac{(C_f+C_y)^2\lceil\frac{T}{T_0}\rceil^{L(1-\alpha)}}{m^{2p}L^{L-1}\lceil\frac{T+1}{T_0}\rceil^{2\alpha}}\Big(\mathcal{V}_l+\frac{\eta C_y^2 T_0}{2(1-\alpha)}\Big)^L
	\\\overset{\rm{Lemma}\ \ref{lemma: log}}{\leq}&
		\eta^2\frac{(C_f+C_y)^2 }{m^{2p}L^{L-1}\lceil\frac{T+1}{T_0}\rceil^{(L+2)\alpha-L}}\Big(\mathcal{V}_l
		+\frac{\eta C_y^2 T_0}{2(1-\alpha)}\Big)^L
	\\\overset{(\ref{eta proof case < 1})}{\leq}&
		\eta^2\frac{(C_f+C_y)^2 T_0^{(L+2)\alpha-L} }{m^{2p}L^{L-1}{(T+1)}^{(L+2)\alpha-L}}\Big((1+2\lambda^2)\left\|\Theta^{(l)}(0)\right\|_2^{2}+\lambda^2\left\|\Theta^{(l)}(0)\right\|_2^{2}\Big)^L
		\\
		\overset{\rm Lemma\ \ref{lemma: T sum}}{\leq}&
			\eta^2\frac{(C_f+C_y)^2 T_0^{(L+2)\alpha-L} 
			(1+3\lambda^2)^L\left\|\Theta^{(l)}(0)\right\|_2^{2L}
			}{m^{2p}L^{L-1}\Big((L+2)\alpha-(L+1)\Big)}\Big(\frac{1}{T^{(L+2)\alpha-(L+1)}}-\frac{1}{(T+1)^{(L+2)\alpha-(L+1)}}\Big)
		\\\overset{(\ref{eta proof case < 1})}{\leq}&\lambda^2\left\|\Theta^{(l)}(0)\right\|_2^{2}\Big(\frac{1}{T^{(L+2)\alpha-(L+1)}}-\frac{1}{(T+1)^{(L+2)\alpha-(L+1)}}\Big)
		\\=&\phi(T)-\phi(T+1).
	\end{aligned}
	\]

	So we have
	\[
	\left\|\Theta^{(l)}{(T+1)}\right\|_2^2\leq
	\mathcal{V}_l-\phi(T+1)+\sum_{t=0}^{T}2\eta_t\psi(t).
	\]
By induction, we obtain
	\begin{equation}\label{lemma proof case I}
	\left\|\Theta^{(l)}{(T)}\right\|_2^2\leq
	\mathcal{V}_l+\sum_{t=0}^{T-1}2\eta_t\sqrt{2\mathcal{L}_n(\Theta{(t)})}\Big(C_y-\sqrt{2\mathcal{L}_n(\Theta{(t)})}\Big).
	\end{equation}
\\\hspace*{\fill}\\
\textbf{II. The case $\alpha=1.$}\ \\
For any $\lambda\in(0,1/\sqrt{3})$ and $\epsilon\in(0,1)$, we choose $\eta$:
\begin{equation}\label{eta proof case = 1}
\begin{aligned}
    \eta\leq\min_{l\in[L+1]}\Bigg\{&\frac{\lambda m^{p}L^{\frac{L-1}{2}}}{(C_f+C_y)\left\|\Theta^{(l)}(0)\right\|_2^{L-1}},
    \frac{2(1-\epsilon)\lambda^2\left\|\Theta^{(l)}(0)\right\|_2^{2}}{C_y^2 T_0},
    \\&\frac{\lambda m^p L^{\frac{L-1}{2}}\sqrt{\epsilon}}{(C_f+C_y)T_0^{\frac{1+\epsilon}{2}}(1+3\lambda^2)^{\frac{L}{2}}\left\|\Theta^{(l)}(0)\right\|_2^{L-1}}
    \Bigg\}.
\end{aligned}
\end{equation}

For the sake of brevity, for any $t\geq 0$, we define:	
\begin{gather*}
\mathcal{V}_l=\Big(1+2\lambda^2\Big)\left\|\Theta^{(l)}{(0)}\right\|_2^2,\\
\phi(t)=        	    \frac{\lambda^2\left\|\Theta^{(l)}{(0)}\right\|_2^2}{(t+1)^{\frac{1}{L+1}}},\\ \psi(t)=\sqrt{2\mathcal{L}_n(\Theta{(t)})}\Big(C_y-\sqrt{2\mathcal{L}_n(\Theta{(t)})}\Big).
\end{gather*}

It's easy to verify $\psi(t)\leq\frac{1}{4}C_y^2\leq\frac{1}{4},\forall t\geq1.$

	With the marks above, we prove the following estimation by induction:
    \[
	\left\|\Theta^{(l)}{(T)}\right\|_2^2\leq
	\mathcal{V}_l-\phi(T)+\sum_{t=0}^{T-1}2\eta_t\psi(t),\ \forall T\in\mathbb{N}_{+}.
	\]
	For $T=1$, with ($\ref{fine grained norm inequation}$) we have:
	\begin{align*}
		\left\|\Theta^{(l)}{(1)}\right\|_2^2&\leq
		\left\|\Theta^{(l)}{(0)}\right\|_2^2+\frac{\eta^2(C_f+C_y)^2}{m^{2p}L^{L-1}}\left\|\Theta^{(l)}{(0)}\right\|_2^{2L}+2\eta\sqrt{2\mathcal{L}_n(\Theta{(0)})}\Big(C_y-\sqrt{2\mathcal{L}_n(\Theta{(0)})}\Big)
		\\&\overset{(\ref{eta proof case = 1})}{\leq}
		(1+\lambda^2)\left\|\Theta^{(l)}{(0)}\right\|_2^2+2\eta\sqrt{2\mathcal{L}_n(\Theta{(0)})}\Big(C_y-\sqrt{2\mathcal{L}_n(\Theta{(0)})}\Big)
		\\&=
		\mathcal{V}_l-\phi(1)+2\eta\psi(0).
	\end{align*}
	
	Assume the inequation holds for $1,\cdots,T$, then for $T+1$ we have:
	\begin{align*}
		\left\|\Theta^{(l)}{(T+1)}\right\|_2^2\leq&
		\phi(T+1)-\phi(T)+
		\Big(\mathcal{V}_l-\phi(T+1)
		+\sum_{t=0}^{T}2\eta_t\psi(t)\Big)\\&
		+\frac{\eta_T^2(C_f+C_y)^2}{m^{2p}L^{L-1}}\Big(\mathcal{V}_l-\phi(T)
		+\sum_{t=0}^{T-1}2\eta_t\psi(t)\Big)^L.
	\end{align*}

where
	\begin{align*}
		&\frac{\eta_T^2(C_f+C_y
		)^2}{m^{2p}L^{L-1}}\Big(\mathcal{V}_l-\phi(T)+\sum_{t=0}^{T-1}2\eta_t\psi(t)\Big)^L
		\\\leq&
		\frac{\eta_T^2(C_f+C_y)^2}{m^{2p}L^{L-1}}\Big(\mathcal{V}_l+\frac{C_y^2}{2}\sum_{t=0}^{T-1}\eta_t\Big)^L
		\\=&
		\eta^2\frac{(C_f+C_y)^2}{m^{2p}L^{L-1}\lceil \frac{T+1}{T_0}\rceil^{2}}\Big(\mathcal{V}_l+\frac{\eta C_y^2}{2}\sum_{t=0}^{T-1}\frac{1}{\lceil\frac{t+1}{T_0}\rceil}\Big)^L
		\\\leq&
   	    \eta^2\frac{(C_f+C_y)^2}{m^{2p}L^{L-1}\lceil\frac{T+1}{T_0}\rceil^{2}}\Big(\mathcal{V}_l+\frac{\eta C_y^2}{2}T_0\sum_{k=1}^{\lceil\frac{T}{T_0}\rceil}\frac{1}{k}\Big)^L
   	    \\
       	    \overset{\rm{Lemma\ }\ref{lemma: T sum}}{\leq}&
   	    \eta^2\frac{(C_f+C_y)^2}{m^{2p}L^{L-1}\lceil\frac{T+1}{T_0}\rceil^{2}}\Big(\mathcal{V}_l+\frac{\eta C_y^2}{2}T_0\big(\log\lceil\frac{T}{T_0}\rceil +1\big)\Big)^L
	\\\overset{\rm{Lemma}\ \ref{lemma: log}}{\leq}&
		\eta^2\frac{(C_f+C_y)^2 }{m^{2p}L^{L-1}\lceil\frac{T+1}{T_0}\rceil^{2}}\Big(\mathcal{V}_l
		+\frac{\eta C_y^2 T_0 L}{2(1-\epsilon)} \lceil\frac{T}{T_0}\rceil^{\frac{1-\epsilon}{L}}\Big)^L
		\\\leq&
			\eta^2\frac{(C_f+C_y)^2 }{m^{2p}L^{L-1}\lceil\frac{T+1}{T_0}\rceil^{1+\epsilon}}\Big(\mathcal{V}_l
		+\frac{\eta C_y^2 T_0 L}{2(1-\epsilon)}\Big)^L	
	\\\overset{(\ref{eta proof case = 1})}{\leq}&
		\eta^2\frac{(C_f+C_y)^2 T_0^{1+\epsilon} }{m^{2p}L^{L-1}{(T+1)}^{1+\epsilon}}\Big((1+2\lambda^2)\left\|\Theta^{(l)}(0)\right\|_2^{2}+\lambda^2\left\|\Theta^{(l)}(0)\right\|_2^{2}\Big)^L
		\\\overset{\rm Lemma\ \ref{lemma: T sum}}{\leq}&
			\eta^2\frac{(C_f+C_y)^2 T_0^{1+\epsilon} 
			(1+3\lambda^2)^L\left\|\Theta^{(l)}(0)\right\|_2^{2L}
			}{m^{2p}L^{L-1}\epsilon}\Big(\frac{1}{T^\epsilon}-\frac{1}{(T+1)^\epsilon}\Big)
		\\\overset{{\rm set\ }\epsilon=\frac{1}{L+1}}{\leq}&
		\eta^2\frac{(C_f+C_y)^2 T_0^{1+\epsilon} 
		(1+3\lambda^2)^L\left\|\Theta^{(l)}(0)\right\|_2^{2L}
		}{m^{2p}L^{L-1}\epsilon}\Big(\frac{1}{T^\epsilon}-\frac{1}{(T+1)^\epsilon}\Big)
		\\\overset{(\ref{eta proof case = 1})}{\leq}&\lambda^2\left\|\Theta^{(l)}(0)\right\|_2^{2}\Big(\frac{1}{T^{\epsilon}}-\frac{1}{(T+1)^{\epsilon}}\Big)
		\\=&\phi(T)-\phi(T+1).
	\end{align*}
So we have
	\[
	\left\|\Theta^{(l)}{(T+1)}\right\|_2^2\leq
	\mathcal{V}_l-\phi(T+1)+\sum_{t=0}^{T}2\eta_t\psi(t).
	\]
By induction, we obtain
	\begin{equation}\label{lemma proof case II}
	\left\|\Theta^{(l)}{(T)}\right\|_2^2\leq
	\mathcal{V}_l+\sum_{t=0}^{T-1}2\eta_t\sqrt{2\mathcal{L}_n(\Theta{(T)})}\Big(C_y-\sqrt{2\mathcal{L}_n(\Theta{(T)})}\Big).
	\end{equation}
\\\hspace*{\fill}\\
Combing (\ref{lemma proof case I}) and (\ref{lemma proof case II}), we have:
\[
	\left\|\Theta^{(l)}{(T)}\right\|_2^2\leq
	\mathcal{V}_l+\sum_{t=0}^{T-1}2\eta_t\sqrt{2\mathcal{L}_n(\Theta{(t)})}\Big(C_y-\sqrt{2\mathcal{L}_n(\Theta{(t)})}\Big),\ \forall l\in[L+1].
\]
Then with Lemma \ref{Estimation of Rademacher Complexity}, we have the Rademacher complexity estimation:
\[
\begin{aligned}
    &{\rm Rad}_n(\mathcal{F}(T))
    \\\leq&
    \frac{C_{L,d}}{m^p\sqrt{n}}\prod\limits_{l=1}^{L+1}\left\|\Theta^{(l)}(T)\right\|_2
    \\\leq&
     \frac{C_{L,d}}{m^p\sqrt{n}}\prod_{l=1}^{L+1}\Bigg((1+3\lambda^2)\left\|\Theta^{(l)}(0)\right\|_2^2+\sum_{t=0}^{T-1}2\eta_t\sqrt{2\mathcal{L}_n(\Theta{(t)})}\Big(C_y-\sqrt{2\mathcal{L}_n(\Theta{(t)})}\Big)\Bigg)^{\frac{1}{2}}.
\end{aligned}
\]

Recalling Lemma \ref{lemma: random initialization}, with probability at least $1-\delta$ we have:
\[
\left\|\Theta^{(l)}(0)\right\|_2^2\leq\kappa^2+\kappa^2\max\Big\{\frac{4}{q(l)}\log(\frac{1}{\delta}),\sqrt{\frac{8}{q(l)}\log(\frac{1}{\delta})\big)}\Big\}=\mathcal{O}\Big(\kappa^2\Big),\  \forall l\in[L+1].
\]

Integrated the proofs above, we have our result:

From (\ref{eta proof case < 1}) and (\ref{eta proof case = 1}), we can choose the initial step size $\eta$ s.t.
\[
\begin{aligned}
    \eta=&\mathcal{O}\Bigg(\frac{ m^{p}L^{\frac{L-1}{2}}}{(C_f+C_y)\kappa^{L-1}},
    \frac{\kappa^2(1-\epsilon)}{C_y^2 T_0},\frac{ m^p L^{\frac{L-1}{2}}\sqrt{\epsilon}}{(C_f+C_y)T_0^{\frac{1+\epsilon}{2}}\kappa^{L-1}}
    \Bigg)
    \\=&
    \mathcal{O}\Bigg(\frac{m^p L^{\frac{L-1}{2}}\sqrt{\epsilon}}{(C_f+C_y)T_0^{\frac{1+\epsilon}{2}}\kappa^{L-1}},
    \frac{\kappa^2(1-\epsilon)}{C_y^2 T_0}
    \Bigg),
\end{aligned}
\]
where $\epsilon\in(0,1)$. 

Then with probability at least $1-\delta$ we have:
\[
    {\rm Rad}_n(\mathcal{F}(T))
    \leq\frac{C_{L,d}}{m^p\sqrt{n}}\Bigg(\mathcal{O}\Big(\kappa^2\Big)+\sum_{t=0}^{T-1}2\eta_t\sqrt{2\mathcal{L}_n(\Theta{(t)})}\Big(C_y-\sqrt{2\mathcal{L}_n(\Theta{(t)})}\Big)\Bigg)^{\frac{L+1}{2}},
\]
so with probability at least $1-2\delta$ we have:
\[
\begin{aligned}
    \mathcal{E}_{\rm gen}(\mathcal{F}(T))
    \lesssim
    \frac{C_{L,d}}{m^p\sqrt{n}}\Bigg(\mathcal{O}\Big(\kappa^2\Big)+\sum_{t=0}^{T-1}2\eta_t\sqrt{2\mathcal{L}_n(\Theta{(t)})}\Big(C_y-\sqrt{2\mathcal{L}_n(\Theta{(t)})}\Big)\Bigg)^{\frac{L+1}{2}}+\sqrt{\frac{\log(1/\delta)}{n}}.
\end{aligned}
\]
Substituting $\delta$ with $\delta/2$, we obtain this theorem.

\end{proof}

\section{Proof Details of SGD}\label{section: proof SGD}
\begin{proof}[\bfseries Proof of Theorem \ref{thm disc SGD 2NN}]\ \\
	We define  $\mathbb{E}_t:=\mathbb{E}_{\gamma^0,\cdots,\gamma^t}=\mathbb{E}_{(\gamma_1^0,\cdots,\gamma_B^0),\cdots,(\gamma_1^t,\cdots,\gamma_B^t)}$ for $t\geq 0$ and $\mathbb{E}_{-1}:=id$. Then we have:
	\[
	\begin{aligned}
		&\mathbb{E}_t\left\|\Theta^{(l)}{(t+1)}\right\|_2^2
		\\=&
		\mathbb{E}_{t-1}\left\|\Theta^{(l)}{(t)}\right\|_2^2+\mathbb{E}_t\left\|\Theta^{(l)}{(t+1)}-\Theta^{(l)}{(t)}\right\|_2^2+2\mathbb{E}_t\left<\Theta^{(l)}{(t)},\Theta^{(l)}{(t+1)}-\Theta^{(l)}{(t)}\right>
		\\=&
		\mathbb{E}_{t-1}\left\|\Theta^{(l)}{(t)}\right\|_2^2+\eta_t^2\mathbb{E}_t\left\|\frac{1}{B}\sum_{i=1}^B\Big(f(\mathbf{x}_{\gamma_i^t};\Theta{(t)})-y_{\gamma_i^t}\Big)\frac{\partial f(\mathbf{x}_{\gamma_i^t};\Theta{(t)})}{\partial \Theta^{(l)}}\right\|_2^2
		\\&\ \ \ \ \ \ \ \ \ \ \ \  -2\eta_t\mathbb{E}_{t}\left<\Theta^{(l)}{(t)},\frac{1}{B}\sum_{i=1}^B\Big(f(\mathbf{x}_{\gamma_i^t};\Theta{(t)})-y_{\gamma_i^t}\Big)\frac{\partial f(\mathbf{x}_{\gamma_i^t};\Theta{(t)})}{\partial \Theta^{(l)}}\right>
		\\\leq&
		\mathbb{E}_{t-1}\left\|\Theta^{(l)}{(t)}\right\|_2^2+\eta_t^2(C_f+C_y)^2\mathbb{E}_{t-1}\sup\limits_{\mathbf{x}}\left\|\frac{\partial f(\mathbf{x};\Theta{(t)})}{\partial \Theta^{(l)}}\right\|_2^2
		\\&\ \ \ \ \ \ \ \ \ \ \ \ \ \ \ -2\eta_t\mathbb{E}_t\Big[\frac{1}{B}\sum_{i=1}^B\Big(f(\mathbf{x}_{\gamma_i^t};\Theta{(t)})-y_{\gamma_i^t}\Big)f(\mathbf{x}_{\gamma_i^t};\Theta{(t)})\Big]
		\\
		\leq&\mathbb{E}_{t-1}
		\left\|\Theta^{(l)}{(t)}\right\|_2^2+\frac{\eta_t^2(C_f+C_y)^2}{m^{2p}L^{L-1}}\mathbb{E}_{t-1}\left\|\Theta^{(l)}{(t)}\right\|_2^{2L}
		\\&\ \ \ \ \ \ \ \ \ \ \ \ \ \ \ -2\eta_t\mathbb{E}_t\Big[\frac{1}{B}\sum_{i=1}^B\Big(f(\mathbf{x}_{\gamma_i^t};\Theta{(t)})-y_{\gamma_i^t}\Big)f(\mathbf{x}_{\gamma_i^t};\Theta{(t)})\Big]
		.
	\end{aligned}
	\]
	Recalling the homogeneity property \ref{multiplicative property}, we have:
	\[
	\begin{aligned}
		&-2\eta_t\mathbb{E}_t\Big[\frac{1}{B}\sum_{i=1}^B\Big(f(\mathbf{x}_{\gamma_i^t};\Theta{(t)})-y_{\gamma_i^t}\Big)f(\mathbf{x}_{\gamma_i^t};\Theta{(t)})\Big]
		\\
		=&-2\eta_t\mathbb{E}_{t-1}\Bigg[\mathbb{E}_{\gamma^t}\Big[\frac{1}{B}\sum_{i=1}^B\Big(f(\mathbf{x}_{\gamma_i^t};\Theta{(t)})-y_{\gamma_i^t}\Big)f(\mathbf{x}_{\gamma_i^t};\Theta{(t)})\Big|\gamma^0,\cdots,\gamma^{t-1}\Big]\Bigg]
		\\
		=&\eta_t\mathbb{E}_{t-1}\Bigg[\mathbb{E}_{\gamma^t}\Big[\frac{1}{B}\sum_{i=1}^B\Big(-4\ell(\mathbf{x}_{\gamma_i^t},y_{\gamma_i^t};\Theta(t))+2(f(\mathbf{x}_{\gamma_i^t};\Theta{(t)})-y_{\gamma_i^t})y_{\gamma_i^t}\Big)\Big|\gamma^0,\cdots,\gamma^{t-1}\Big]\Bigg]
		\\
		=&\eta_t\mathbb{E}_{t-1}\Big[-4\mathcal{L}_n(\Theta(t))+\frac{2}{n}\sum_{i=1}^n\Big(f(\mathbf{x}_i;\Theta{(t)})-y_i\Big)y_i\Big]
		\\
		\leq
		&\eta_t\mathbb{E}_{t-1}\Big[-4\mathcal{L}_n(\Theta(t))+2C_y\sqrt{2\mathcal{L}_n(\Theta(t))}\Big]
		\\
		=&
		2\eta_t\mathbb{E}_{t-1}\Big[\sqrt{2\mathcal{L}_n(\Theta{(t)})}\Big(C_y-\sqrt{2\mathcal{L}_n(\Theta{(t)})}\Big)\Big].
	\end{aligned}
	\]
	Combining the two formulations above, we obtain:
	\[
	\begin{aligned}
	\mathbb{E}_{t}\Big[\left\|\Theta^{(l)}{(t+1)}\right\|_2^2\Big]
	\leq&
	\mathbb{E}_{t-1}\Big[\left\|\Theta^{(l)}{(t)}\right\|_2^2\Big]+\frac{\eta_t^2(C_f+C_y)^2}{m^{2p}L^{L-1}}\mathbb{E}_{t-1}\Big[\left\|\Theta^{(l)}{(t)}\right\|_2^{2L}\Big]
	\\&+2\eta_t\mathbb{E}_{t-1}\Big[\sqrt{2\mathcal{L}_n(\Theta{(t)})}\Big(C_y-\sqrt{2\mathcal{L}_n(\Theta{(t)})}\Big)\Big].
	\end{aligned}
	\]
The remaining proof method is close to the proof of Theorem \ref{thm disc GD 2NN}, we only need to replace some constants with their expectations. We give the framework of the proof below.
\\\hspace*{\fill}\\
\textbf{I. The case $\alpha\in(\frac{L+1}{L+2},1).$}

For any $\lambda\in(0,1/\sqrt{3})$, we choose $\eta$:
\[
\begin{aligned}
    \eta\leq\min_{l\in[L+1]}\Bigg\{&\frac{\lambda m^{p}L^{\frac{L-1}{2}}}{(C_f+C_y)\left\|\Theta^{(l)}(0)\right\|_2^{L-1}},
    \frac{2(1-\alpha)\lambda^2\left\|\Theta^{(l)}(0)\right\|_2^{2}}{C_y^2 T_0},
    \\&\frac{\lambda m^p L^{\frac{L-1}{2}}\sqrt{(L+2)\alpha-(L+1)}}{(C_f+C_y)T_0^{\frac{(L+2)\alpha-L}{2}}(1+3\lambda^2)^{\frac{L}{2}}\left\|\Theta^{(l)}(0)\right\|_2^{L-1}}
    \Bigg\}.
\end{aligned}
\]

For the sake of brevity, for any $t\geq 1$, we define:
	\[
	\begin{gathered}
\mathcal{V}_l=\Big(1+2\lambda^2\Big)\left\|\Theta^{(l)}{(0)}\right\|_2^2,\\
\phi(t)=        	    \frac{\lambda^2\left\|\Theta^{(l)}{(0)}\right\|_2^2}{(t+1)^{(L+2)\alpha-(L+1)}},\\ \psi(t)=\mathbb{E}_{t-1}\Big[\sqrt{2\mathcal{L}_n(\Theta{(t)})}\Big(C_y-\sqrt{2\mathcal{L}_n(\Theta{(t)})}\Big)\Big].
\end{gathered}
\]

	Then we have the following estimation by induction, same as the proof of Theorem \ref{thm disc GD 2NN}.
	\[
	\mathbb{E}_{T-1}\Big[\left\|\Theta^{(l)}{(T)}\right\|_2^2\Big]\leq
	\mathcal{V}_l-\phi(T)+\sum_{t=0}^{T-1}2\eta_t\psi(t),\ \forall T\in\mathbb{N}_{+}.
	\]
	\\\hspace*{\fill}\\
\textbf{II. The case $\alpha=1.$}\ 

For any $\lambda\in(0,1/\sqrt{3})$ and $\epsilon\in(0,1)$, we choose $\eta$:
\[
\begin{aligned}
    \eta\leq\min_{l\in[L+1]}\Bigg\{&\frac{\lambda m^{p}L^{\frac{L-1}{2}}}{(C_f+C_y)\left\|\Theta^{(l)}(0)\right\|_2^{L-1}},
    \frac{2(1-\epsilon)\lambda^2\left\|\Theta^{(l)}(0)\right\|_2^{2}}{C_y^2 T_0},
    \\&\frac{\lambda m^p L^{\frac{L-1}{2}}\sqrt{\epsilon}}{(C_f+C_y)T_0^{\frac{1+\epsilon}{2}}(1+3\lambda^2)^{\frac{L}{2}}\left\|\Theta^{(l)}(0)\right\|_2^{L-1}}
    \Bigg\}.
\end{aligned}
\]

For the sake of brevity, for any $t\geq 0$, we define:
	\[
	\begin{gathered}
\mathcal{V}_l=\Big(1+2\lambda^2\Big)\left\|\Theta^{(l)}{(0)}\right\|_2^2,\\
\phi(t)=        	 \frac{\lambda^2\left\|\Theta^{(l)}{(0)}\right\|_2^2}{(t+1)^{\frac{1}{L+1}}},\\ \psi(t)=
\mathbb{E}_{t-1}\Big[\sqrt{2\mathcal{L}_n(\Theta{(t)})}\Big(C_y-\sqrt{2\mathcal{L}_n(\Theta{(t)})}\Big)\Big].
\end{gathered}
	\]
	Then we have the following estimation by induction, same as the proof of Theorem \ref{thm disc GD 2NN}.
	\[
	\mathbb{E}_{T-1}\Big[\left\|\Theta^{(l)}{(T)}\right\|_2^2\Big]\leq
	\mathcal{V}_l-\phi(T)+\sum_{t=0}^{T-1}2\eta_t\psi(t),\ \forall T\in\mathbb{N}_{+}.
	\]
Combining the two cases above, we obtain the result:
	\[
 	\mathbb{E}_{T-1}\Big[\left\|\Theta^{(l)}(T)\right\|_2^2\Big]\leq
	\mathcal{V}_l+\sum_{t=0}^{T-1}2\eta_t\psi(t),\ \forall T\in\mathbb{N}_{+}.
	\]
From Markov Inequation, we have
    \[
    \begin{aligned}
        \mathbb{P}\Big(\frac{1}{L+1}\sum\limits_{l=1}^{L+1}
        \left\|\Theta^{(l)}(T)\right\|_2^2\geq\frac{1+\rho}{L+1}\sum\limits_{l=1}^{L+1}\mathbb{E}\left\|\Theta^{(l)}(T)\right\|_2^2\Big)\leq
        \frac{1}{1+\rho}.
    \end{aligned}
    \]
    So with probability at least $\frac{\rho}{1+\rho}$ we have:
\[
\left\|\Theta^{(l)}(0)\right\|_2^2\leq (1+\rho)\mathbb{E}\Big[\left\|\Theta^{(l)}(0)\right\|_2^2\Big],\  \forall l\in[L+1].
\]
So we can choose the initial step size $\eta$ s.t.
\[
\begin{aligned}
    \eta=&\mathcal{O}\Bigg(\frac{ m^{p}L^{\frac{L-1}{2}}}{(C_f+C_y)\kappa^{L-1}},
    \frac{\kappa^2(1-\epsilon)}{C_y^2 T_0},\frac{ m^p L^{\frac{L-1}{2}}\sqrt{\epsilon}}{(C_f+C_y)T_0^{\frac{1+\epsilon}{2}}\kappa^{L-1}}
    \Bigg)
    \\=&
    \mathcal{O}\Bigg(\frac{m^p L^{\frac{L-1}{2}}\sqrt{\epsilon}}{(C_f+C_y)T_0^{\frac{1+\epsilon}{2}}\kappa^{L-1}},
    \frac{\kappa^2(1-\epsilon)}{C_y^2 T_0}
    \Bigg),
\end{aligned}
\]
where $\epsilon\in(0,1)$. 

Then with Lemma \ref{Estimation of Rademacher Complexity}, with probability $\frac{\rho}{1+\rho}-\delta$ we have the Rademacher complexity estimation:
\[
\begin{aligned}
    &{\rm Rad}_n(\mathcal{F}(T))
    \\\leq&
    \frac{C_{L,d}}{m^p\sqrt{n}}\prod\limits_{l=1}^{L+1}\left\|\Theta^{(l)}(T)\right\|_2
    \\\leq&
     \frac{C_{L,d}}{m^p\sqrt{n}}\prod_{l=1}^{L+1}\Bigg(\mathcal{O}\Big((1+\rho)\kappa^2\Big) +\sum_{t=0}^{T-1}2(1+\rho)\eta_t\mathbb{E}\Big[\sqrt{2\mathcal{L}_n(\Theta{(t)})}\Big(C_y-\sqrt{2\mathcal{L}_n(\Theta{(t)})}\Big)\Big]\Bigg)^{\frac{1}{2}}.
\end{aligned}
\]

So from Lemma \ref{lemma: gen and Rad}, with probability at least $\frac{\rho}{1+\rho}-2\delta$, we obtain:
\[
\begin{aligned}
    \mathcal{E}_{\rm gen}(\mathcal{F}(T))
    \lesssim
    \frac{C_{L,d}}{m^p\sqrt{n}}\Bigg(&\mathcal{O}\Big((1+\rho)\kappa^2\Big)+\sum_{t=0}^{T-1}2(1+\rho)\eta_t\mathbb{E}\Big[\sqrt{2\mathcal{L}_n(\Theta{(t)})}\Big(C_y-\sqrt{2\mathcal{L}_n(\Theta{(t)})}\Big)\Big]\Bigg)^{\frac{L+1}{2}}
    \\+&\sqrt{\frac{\log(1/\delta)}{n}}.
\end{aligned}
\]
Substituting $\delta$ with $\delta/2$, we obtain this theorem.
\end{proof}

\newpage

\section{Extension to Other Power-type Loss}\label{section: proof loss extension}

Our results apply directly to the following power-type loss:
\[\ell(f,g)=\frac{|f-g|^{\alpha}}{\alpha},\ (\alpha\geq2 \text{ and } \alpha\in\mathbb{N}).\]

For Theorem \ref{continuous GD thm} (GF),
we only need to replace 
${\rm\sf{CL}}(T)=\int_{0}^T2\sqrt{2\mathcal{L}_n(\Theta{(t)})}\Big(C_y-\sqrt{2\mathcal{L}_n(\Theta{(t)})}\Big)\mathrm{d}t$ with ${\rm\sf CL}(T):=\int_{0}^T 2\Big(\alpha\mathcal{L}_n(\Theta(t))\Big)^{\frac{\alpha-1}{\alpha}}\Big(C_y-(\alpha\mathcal{L}_n(\Theta(t)))^{\frac{1}{\alpha}}\Big)\mathrm{d}t$.
Because we only need to substitute the estimate \eqref{equ: extension discussion} with
\begin{align*}
	\frac{\mathrm{d} \left\|\Theta^{(l)}(t)\right\|^2}{\mathrm{d} t}=&
	-2\left<\Theta^{(l)}(t),\frac{\partial \mathcal{L}_n(\Theta(t))}{\partial \Theta^{(l)}}\right>
	\\=&
	-\frac{2}{n}\sum_{i=1}^n\left|f(\mathbf{x}_i;\Theta(t))-y_i\right|^{\alpha-1}\text{sgn}(f(\mathbf{x}_i;\Theta(t))-y_i)\left<\Theta^{(l)}(t),\frac{\partial f(\mathbf{x}_i;\Theta(t))}{\partial \Theta^{(l)}}\right>
	\\=&
	-\frac{2}{n}\sum_{i=1}^n\left|f(\mathbf{x}_i;\Theta(t))-y_i\right|^{\alpha-1}\text{sgn}(f(\mathbf{x}_i;\Theta(t))-y_i)f(\mathbf{x}_i;\Theta{(t)})
	\\=&
	-2\alpha\mathcal{L}_n(\Theta{(t)})-\frac{2}{n}\sum_{i=1}^n\left|f(\mathbf{x}_i;\Theta(t))-y_i\right|^{\alpha-1}\text{sgn}(f(\mathbf{x}_i;\Theta(t))-y_i)y_i
	\\
	\overset{\text{Holder inequality}}{\leq}&
	-2\alpha\mathcal{L}_n(\Theta{(t)})+\frac{2}{n}\Big(\sum_{i=1}^n\left|f(\mathbf{x}_i;\Theta{(t)})-y_i\right|^\alpha\Big)^{\frac{\alpha-1}{\alpha}}\Big(\sum_{i=1}^n y_i^\alpha\Big)^{\frac{1}{\alpha}}
	\\\leq&-2\alpha\mathcal{L}_n(\Theta{(t)})+2C_y\Big(\alpha\mathcal{L}_n(\Theta{(t)})\Big)^{\frac{\alpha-1}{\alpha}}
	\\=&
	2\Big(\alpha\mathcal{L}_n(\Theta(t))\Big)^{\frac{\alpha-1}{\alpha}}\Big(C_y-(\alpha\mathcal{L}_n(\Theta(t)))^{\frac{1}{\alpha}}\Big),\ \forall l\in[L+1],
\end{align*}
then we can derive similar bounds.

the Cauchy inequality (in the estimate of the second part) with Holder inequality $(p=\frac{\alpha}{\alpha-1},\ q=\alpha)$. 

In the same way, for Theorem \ref{thm disc GD 2NN} (GD) and \ref{thm disc SGD 2NN} (SGD), we only need to replace ${\rm\sf{CL}}(T)$ \eqref{equ: disc CT} with ${\rm\sf CL}(T):=\sum\limits_{t=0}^T 2\eta_t\Big(\alpha\mathcal{L}_n(\Theta(t))\Big)^{\frac{\alpha-1}{\alpha}}\Big(C_y-(\alpha\mathcal{L}_n(\Theta(t)))^{\frac{1}{\alpha}}\Big)$.

\newpage

\section{Experiment Details}\label{section: appendix experiment}
\begin{itemize}
    \item For the function regression problem (Fig \ref{figure: function regression bound}), we use the following settings. Model: 3-depth FNN $3\to1500\to1500\to1$ with normalization $1/m^p=1/4$. Dataset: $f^{*}(\mathbf{x})=(x_1+x_2^2+\sin(\pi x_3))/(1.25+\pi^2/4)$  $(\mathbf{x}\in[-1/\sqrt{3},1/\sqrt{3}]^3)$, $n_{tr} = 200000$ and $n_{te}=100000$. Algorithm: SGD, batch size=$2000$, the learning rate is chosen as Theorem \ref{thm disc SGD 2NN}: $\eta_t=\eta/\lceil\frac{t+1}{T_0}\rceil^\alpha$, $\alpha=0.67$, $T_0=1$ and $\eta = 0.1$, the scale of random initialization: $\kappa=4$.
    \item
For MNIST classification problem (Fig \ref{figure: MNIST bound}), we use the following settings. Model: 4-depth FNN $28\times28\to256\to64\to32\to2$ with normalization $1/m^p=1/4$. Dataset with normalization: MNIST (label=0,1) with normalizing $\left\|\mathbf{x}\right\|_2\leq1$, $\left\| y\right\|_2\leq C_y={1}/{4}$. Algorithm: SGD, batch size=$2000$, the learning rate is chosen as Theorem \ref{thm disc SGD 2NN}: $\eta_t=\eta/\lceil\frac{t+1}{T_0}\rceil^\alpha$, $\alpha=1$, $T_0=120$ and $\eta = 0.2$, the scale of random initialization: $\kappa=2$.

\item For the experiments about effect of hyperparamerters (Fig \ref{figure: bound-lr} and \ref{figure: bound-width}), we consider the function regression problem mentioned in the first experiment with two-layer neural network $1\to {\rm width} \to 1$ with normalization $1/m^p=1/4$. We adopt single variable method to research effects on our bounds with different widths and learning rates separately. Algorithm: SGD, batch size=$2000$, the learning rate is chosen as Theorem \ref{thm disc SGD 2NN}: $\eta_t=\eta/\lceil\frac{t+1}{T_0}\rceil^\alpha$, $\alpha=0.67$, $T_0=1$ and $\eta = 0.1$, the scale of random initialization: $\kappa=4$.

\item For large-scale experiments (Table \ref{table: large-scale}), we use the following settings. Dataset: CIFAR-10. Model: standard VGG networks (without batch normalization). Algorithm: SGD, batch size=100, the learning rate is chosen as Theorem \ref{thm disc SGD 2NN}: $\eta_t=\eta/\lceil\frac{t+1}{T_0}\rceil^\alpha$, $\alpha=1$, $T_0=100$, and stop criterion=$10^6$ iterations.
    {(I)} VGG-16; Cifar-10 (label=0,1) with different proportion of label noise $0$\%, $20$\%, $50$\%, $80$\%, $100$\%.
{(II)} VGG-16; subset of Cifar-10 with different number of class: $2$, $5$, $8$, $10$.
{(III)} Different network sizes VGG-11, VGG-13, VGG-16, VGG-19; on a subset of Cifar-10 (label=0, 1).
\end{itemize}

\end{document}